\documentclass[11pt]{article}

\usepackage[utf8]{inputenc}
\usepackage[margin=1in]{geometry}
\usepackage{mathtools}
\usepackage{amsmath}
\usepackage{amsfonts}
\usepackage{amssymb}
\usepackage{textcomp}
\usepackage{pgfplots}
\usepackage{graphicx}
\usepackage{wrapfig}
\usepackage[]{youngtab}
\usepackage{amsthm}
\usepackage{tikz}
\usepackage{tikz-cd}
\usepackage{stmaryrd}
\usepackage{enumerate}
\usepackage{bbm}

\usepackage{eucal}
\usepackage{mathrsfs}
\pgfplotsset{width=10cm,compat=1.9}
\usepackage{natbib}
\usepackage{hyperref}
\hypersetup{
    colorlinks=true,
    linkcolor=blue,
    citecolor=blue
}

\newtheorem{theorem}{Theorem}
\newtheorem{lemma}[theorem]{Lemma}
\newtheorem{proposition}[theorem]{Proposition}

\newtheorem{corollary}[theorem]{Corollary}
\newtheorem{definition}[theorem]{Definition}

\newcommand{\bx}{{\boldsymbol{x}}}

\newcommand{\cF}{\mathcal{F}}
\newcommand{\cX}{\mathcal{X}}
\newcommand{\cY}{\mathcal{Y}}

\newcommand{\rr}{\mathbb{R}}
\newcommand{\pp}{\mathbb{P}}
\newcommand{\ee}{\mathbb{E}}
\newcommand{\T}{\mathbb{T}}
\newcommand{\R}{\mathcal{R}}
\newcommand{\F}{\mathcal{F}}
\newcommand{\fat}{\mathsf{fat}}

\newcommand{\z}{{\boldsymbol{z}}}
\newcommand{\vbf}{{\boldsymbol{v}}}
\newcommand{\norm}[1]{\left|\left| #1 \right|\right|}
\newcommand{\abs}[1]{\left| #1 \right|}
\DeclareMathOperator{\supp}{supp}
\newcommand{\lp}{\left(}
\newcommand{\rp}{\right)}
\renewcommand{\vec}[1]{\overrightarrow{#1}}
\newcommand{\reals}{\mathbb{R}}
\newcommand{\pred}{\widehat{y}}
\newcommand{\En}{\mathbb{E}}

\renewcommand{\epsilon}{\varepsilon}

\title{Majorizing Measures, Sequential Complexities, and Online Learning}
\author{Adam Block \\ MIT \and Yuval Dagan \\ MIT \and Alexander Rakhlin \\ MIT}

\date\today

\newenvironment{keywords}
{
\textbf{Keywords:}
}
{}

\begin{document}

\maketitle

\begin{abstract}%
  We introduce the technique of generic chaining and majorizing measures for controlling sequential Rademacher complexity. We relate majorizing measures to the notion of fractional covering numbers, which we show to be dominated in terms of sequential scale-sensitive dimensions in a horizon-independent way, and, under additional complexity assumptions establish a tight control on worst-case sequential Rademacher complexity in terms of the integral of sequential scale-sensitive dimension. Finally, we establish a tight contraction inequality for worst-case sequential Rademacher complexity.   The above constitutes the resolution of a number of outstanding open problems in extending the classical theory of empirical processes to the sequential case, and, in turn, establishes sharp results for online learning.
\end{abstract}

\begin{keywords}%
  Online learning, Majorizing Measures, Sequential Rademacher Complexity, Chaining
\end{keywords}

\section{Introduction}

    One of the primary goals of learning theory is to understand how the sample complexity of learning depends on the complexity of the model.  Classically, much of the research focused on the case where data are drawn independently from some population distribution. The mismatch between sample and population inevitably leads to questions of uniform convergence. To answer such questions, the theory of empirical processes was developed, with seminal papers establishing non-asymptotic rates of convergence in terms of covering numbers, chaining \citep{dudley1973}, VC dimension \citep{vapnik1971uniform}, and scale-sensitive combinatorial parameters  \citep{bartlett1996fat,kearns1994efficient}. 
   
    A central notion of complexity is the (empirical) Rademacher complexity \citep{gine1984some}, defined for a data set $z_1, \dots, z_n\in {\mathcal Z}$ and a real-valued function class $\F$ on ${\mathcal Z}$ as
    \begin{equation}
    \label{eq:emp_rad}
        \widehat{\R}_n(\F) = \ee_\epsilon\left[\sup_{f \in \F} \frac 1n \sum_{t = 1}^n \epsilon_t f(z_t) \right],
    \end{equation}
    where $\epsilon_t$ are independent Rademacher random variables. 
    It is a classical result that $\widehat{\R}_n(\F)$ determines the rate of convergence of uniform laws of large numbers over the function class $\F$. Sharp upper and lower bounds on sample complexity of prediction and estimation have been established via Rademacher complexity and its localized versions \citep{bartlett2002rademacher}.

    Just as important as the introduction of these notions of complexity has been the development of the relationships among them.  In particular, it is a result of Dudley \citep{dudley1973} that the Rademacher complexity is controlled by integrating a function of the covering number and this, in turn, can be bounded by the integral of the square root of the scale-sensitive combinatorial dimension, as shown by \cite{mendelson2003entropy,rudelson2006combinatorics}.  On the other hand, closely related Gaussian averages, obtained by replacing Rademacher with Gaussian random variables in \eqref{eq:emp_rad}, are known to be tightly (up to a constant) controlled via generic chaining and majorizing measures \citep{fernique1975regularite,talagrand1996majorizing}, and may be strictly smaller than the bounds provided by Dudley's chaining technique.  It is fair to say that the foundations of statistical learning when the data are i.i.d. and presented all at once are fairly well-understood.
    
    In contrast to statistical learning, in online learning the data arrive sequentially. The learner is tested on the new datum, observes the outcome, and updates the model with the new information. In the common formulation of the online learning problem, the sequence as treated as non-stochastic (individual), or even adaptively and adversarially chosen \citep{cesa2006prediction}. For online binary classification, the fundamental result of \cite{littlestone1988learning} (and the agnostic generalization of \cite{ben2009agnostic}) characterizes learnability in terms of finiteness of Littlestone dimension, the counterpart to VC dimension in statistical learning. Motivated by these results, \cite{rakhlin2010online} developed \textit{sequential} analogues of Rademacher averages, covering numbers, and scale-sensitive dimensions, and showed that sample complexity of online learning can be related to these quantities. For example, for online supervised learning with indicator or absolute value loss, the minimax regret $V_n(\F)$, defined as the average loss of an algorithm minus the average prediction loss of the best model in $\F$, is---up to a multiplicative factor of $2$---equal to the worst-case sequential Rademacher complexity of $\F$ \citep{rakhlin2010online}, a quantity we shall define below. Furthermore, in parallel to the close relationship between statistical learning and uniform laws of large numbers, the sample complexity of online learning can be viewed through the lens of uniform \textit{martingale} laws of large numbers \citep{rakhlin2015sequential}. 
    
    In recent years, sequential complexities have been used in an increasing range of topics, including private learning 
    \citep{alon2019private,bun2020equivalence,jung2020equivalence,ghazi2020sample}, adversarial robustness of sampling streaming data \citep{alon2021Adversarial}, denoising in autoregressive models \citep{hall2016inference,foster2020learning}, 
    contextual bandits \citep{foster2020beyond}, among others. 
    
    While many of the parallels between the classical results in empirical process theory and their sequential (or martingale) counterparts have been established, a number of the fundamental relationships and techniques (such as generic chaining, a sharp contraction principle) are still missing. This paper addresses some of these gaps.

    Before stating our contributions, we define the main object of study, sequential Rademacher averages.  We are primarily interested in certain dyadic martingales. As before, let $\epsilon_1,\ldots,\epsilon_n$ be independent Rademacher random variables, and let $\z_t(\epsilon)_{1 \leq t \leq n}$ be a predictable process with respect to the filtration $\mathcal{A}_t = \sigma(\epsilon_1, \dots, \epsilon_t)$; equivalently, we can think of $\z$ as a complete binary tree of depth $n$ with each vertex $\z_t(\epsilon)$ labelled by an element of $\mathcal{Z}$ according to the path $\epsilon_1, \dots, \epsilon_{t-1}$  (see the bottom of this section for complete definitions).  Given a real-valued function class $\F$ on $\mathcal{Z}$, we follow \citep{rakhlin2010online} and define the sequential Rademacher complexity as
    \begin{equation}
        \R_n(\F, \z) = \ee_{\epsilon}\left[\sup_{f \in \F} \frac 1n \sum_{t = 1}^n \epsilon_t f(\z_t(\epsilon)) \right]
    \end{equation}
    We note that if $\z$ is a tree whose labels are constant as a function of $t$, i.e., $\z_t(\epsilon)$ is independent of $\epsilon$, we recover the classical notion of Rademacher complexity in \eqref{eq:emp_rad}.  Much of \citep{rakhlin2010online} is devoted to demonstrating that this is, in some sense, the ``correct" extension to the online case and this paper is primarily concerned with developing upper bounds for this quantity in terms of various complexity measures of $\F$.

The structure of the paper is given by the order of the following contributions:
\begin{enumerate}
    \item We introduce a notion of majorizing measure in the online setting and show that the sequential Rademacher complexity is bounded:
    \begin{equation}
        \R_n(\F, \z) ~\lesssim~ \inf_{\supp \mu \subseteq \F(\z)} \sup_{\substack{\vbf \in \F(\z) \\ \epsilon \in \{\pm 1\}^n}}\frac 1{\sqrt n}\int_0^1 \sqrt{\log \frac 1{\mu(B_\delta(\vbf,\epsilon))}} d \delta
    \end{equation}
    The upper bound can be viewed as an analogue of the corresponding upper bound via majorizing measures \citep{fernique1975regularite,talagrand1996majorizing}.
    In proving the above, we introduce a new concentration inequality for martingales. 
    
    \item We extend the notion of fractional cover introduced by \cite{alon2021Adversarial} to real-valued function classes and show that this notion of complexity controls the upper bound given by a majorizing measure.
    
    \item We prove a sequential analogue of the results of \cite{mendelson2003entropy} for fractional covers, showing that
    \begin{equation}
        N'(\F, \delta) ~\lesssim~ \lp \frac{C}{\delta} \rp^{\fat_{c \delta}(\F)}
    \end{equation}
    thereby resolving an open question raised in \citep{rakhlin2015sequential}. Here $\fat$ is a sequential scale-sensitive dimension, defined below.
    
    \item We prove a sequential analogue of the results of \cite{rudelson2006combinatorics}, showing that 
    \begin{equation}
        \R_n(\F) ~\lesssim~ \int_0^1 \sqrt{\fat_\delta(\F)} d \delta
    \end{equation}
    where $\R_n(\F) = \sup_{\z} \R_n(\F, \z)$ is the worst-case Rademacher complexity over all trees $\z$.  For sufficiently simple classes $\F$, we show that this inequality can be reversed and apply the result to prove a dimension-independent contraction inequality for sequential Rademacher complexity, thereby resolving another open problem from \citep{rakhlin2015sequential}.
\end{enumerate}

\paragraph{Notation.}
%We fix some notation for the remainder of the paper.  
Bold lower-case letters will denote complete binary trees. Given a depth $n$ binary tree $\vbf$, denote by $\epsilon \in \{\pm 1 \}^n$ a path from the root to a leaf of the tree, where $\epsilon_t = -1$ signifies that node $t+1$ of the path is the left child of node $t$. A \emph{real-valued tree} is a tree labelled by real numbers. If $\vbf,\vbf'$ are real-valued trees, define by $\vbf(\epsilon)=(\vbf_1(\epsilon),\dots,\vbf_n(\epsilon))$ the values of $\vbf$ on the path $\epsilon$, and let
\begin{equation}
    \norm{\vbf(\epsilon) - \vbf'(\epsilon)}^2 = \sum_{t = 1}^n \lp \vbf_t(\epsilon) - \vbf_t'(\epsilon)\rp^2.
\end{equation}
Given a real-valued tree $\vbf$, we let $\T_\vbf$ be the associated tree process, the random variable
\begin{equation}
    \T_\vbf(\epsilon) = \sum_{t = 1}^n \epsilon_t \vbf_t( \epsilon).
\end{equation}
If $\z$ is a $\mathcal{Z}$-valued tree that is clear from context and $f$ is a real-valued function on $\mathcal{Z}$, we often abbreviate $\T_f = \T_{f(\z)}$ and write $\F(\z)=\{f\circ \z: f\in\F\}$.  For two quantities $a, b$, we say that $a \lesssim b$ if there exists a universal constant $C$, independent of the quantities that determine $a,b$, such that $a \leq C b$.

\section{Majorizing Measures and Upper Bounds}

In the classical setting, one of the most fundamental quantities that controls the size of the supremum of a stochastic process is the covering number of the underlying index set, i.e., the minimal number of points such that every member of the index set is within $\delta$ of one of the chosen points.  In \citep{rakhlin2015sequential}, this definition is extended to the sequential setting:
\begin{definition}
  Let $\z$ be a $\mathcal{Z}$-valued tree of depth $n$ and let $\F\subset \rr^{\mathcal{Z}}$.  A covering at scale $\delta$ is a set of binary real-valued trees $\vbf^1, \cdots, \vbf^N$ such that for all $f \in \F$ and all $\epsilon \in \{ \pm 1 \}^n$, there is some $j$ such that
  \begin{equation}
      \norm{f(\z(\epsilon)) - \vbf^j(\epsilon) }^2\leq n\delta^2.
  \end{equation}
  The covering number $N(\F, \delta, \z)$ of a class $\F$ for a tree $\z$ at scale $\delta$ is the minimal $N$ such that there exists a $\delta$-covering of size $N$.
\end{definition}
As noted in \citep{rakhlin2015martingale,rakhlin2010online}, we can recover the classical definition by restricting to trees $\z$, whose values are path-independent.  In analogy with the statistical learning regime, we may apply the chaining technique to bound the sequential Rademacher complexity in terms of these covering numbers; in fact, for uniformly bounded function classes, \cite[Theorem 3]{rakhlin2015sequential} guarantees
\begin{equation}\label{eq1}
    \R_n(\F, \z) ~\lesssim~ \inf_{\alpha > 0} \lp \alpha + \frac {1}{\sqrt n} \int_\alpha^1 \sqrt{\log N(\F,\delta, \z)} d \delta\rp.
\end{equation}
Looking back to the classical case, we note that the generic chaining and majorizing measure approaches \citep{talagrand2014upper,talagrand1996majorizing,fernique1975regularite} would suggest that the bound in \eqref{eq1} is not tight in all cases. In the classical setting, generic chaining approaches give significant improvements over the Dudley entropy integral, as can be seen, for example, from the case of ellipsoids in Hilbert space \citep[Sec 2.5]{talagrand2014upper}. Such improvements go well beyond logarithmic factors. Furthermore,  \cite{rakhlin2015sequential} note that a clean, $n$-independent upper bound for the sequential covering number in terms of the sequential scale-sensitive dimension in the style of \citep{mendelson2003entropy} is not yet known, resulting in additional looseness when using \eqref{eq1}.  

We now introduce a generalization of majorizing measures to the sequential case.
\begin{definition}
	For a collection of $[0,1]$-valued binary trees $\mathcal{V}$ of depth $n$, for $\vbf \in \mathcal{V}$ and for a fixed path $\epsilon$, let $B_\delta(\vbf, \epsilon)$ be the set of trees $\vbf'$ in $\mathcal{V}$ such that $\norm{\vbf(\epsilon) - \vbf'(\epsilon)}^2 \leq n \delta^2$.  For a measure $\mu$ on the set of $[0,1]$-valued trees of depth $n$, and $0 \leq \alpha \leq 1$, let
	\begin{equation}\label{eq:majorizingmeasure}
	I_\mu^\alpha(\vbf, \epsilon) = \alpha + \frac 1{\sqrt n} \int_\alpha^1 \sqrt{\log \frac 1{\mu(B_\delta(\vbf,\epsilon))}} d \delta.
	\end{equation}
	We let
	\begin{align}
	    I_\mu^\alpha = \sup_{\substack{\vbf \in \mathcal{V} \\ \epsilon \in \{\pm 1\}^n}} I_\mu^\alpha(\vbf,\epsilon) &&
	    I_{\F,\z}^\alpha = \inf_{\supp \mu \subseteq \F(\z)} I_\mu^\alpha && I_\F^\alpha = \sup_\z I^\alpha_{\F, \z}
	\end{align}
\end{definition}
To see that this definition extends the classical one, consider the special case where $\mathcal{V}$ is restricted to contain only trees that are constant on vertices of equal depth, corresponding to the batch setting; thus, the path $\epsilon$ in the above definition becomes irrelevant and we recover the classical notion of a majorizing measure \citep[p. 177]{talagrand2014upper}.  This notion of complexity is not very useful unless it actually controls the size of our stochastic process; fortunately, we have:
\begin{theorem}\label{thm:majorizingchaining}
    Let $\F$ be a separable $[0,1]$-valued function class on $\mathcal{Z}$.  Let $\z$ be a $\mathcal{Z}$-valued binary tree of depth $n$.  Then with probability at least $1 - \rho$ over $\epsilon$ and for any $\alpha \geq 0$
    \begin{equation}
        \sup_{f \in \F} \lp\frac 1n \T_{f(\z)}(\epsilon)\rp  ~\lesssim~  I_{\F, \z}^\alpha + \sqrt{\frac{ \log \frac 1\rho}n}
    \end{equation}
    In particular, the sequential Rademacher complexity is bounded by
    \begin{equation}
        \R_n(\F, \z) ~\lesssim~ \inf_{\alpha > 0} \lp \alpha + \inf_{\supp \mu \subset \F(\z)} \sup_{\substack{\vbf \in \F(\z) \\ \epsilon \in \{\pm 1\}^n}}\frac 1{\sqrt n}\int_\alpha^1 \sqrt{\log \frac 1{\mu(B_\delta(\vbf,\epsilon))}} d \delta\rp
    \end{equation}
\end{theorem}
The proof of Theorem \ref{thm:majorizingchaining} is somewhat technical and thus deferred to Appendix \ref{app:chaining}.  The proof rests on the following new martingale concentration inequality that can be viewed as a stronger version of \citep[Lemma 2]{bartlett2008high}:
\begin{lemma}\label{lem:newmartingaleconcentration}
  Let $\vbf$ be a $[-1,1]$-labelled binary tree of depth $n$ and $\epsilon_1, \dots, \epsilon_n$ independent Rademacher variables.  Then there is a constant $C$ such that with probability at least $1 - \rho$ over the $\epsilon_t$,
  \begin{equation}
      \frac{\abs{\sum_{t = 1}^n \epsilon_t \vbf_t(\epsilon)}}{\norm{\vbf(\epsilon)} \sqrt{\log \frac 1\rho + \log\log\frac{e\sqrt{n}}{\norm{\vbf(\epsilon)}}}} \leq C 
  \end{equation}
\end{lemma}
For the probabilist, an equivalent rephrasing of Lemma \ref{lem:newmartingaleconcentration} is to let $M_t$ be a martingale sequence whose differences are conditionally symmetric and bounded in absolute value by $1$.  Let $[M]_t$ be its quadratic variation process.  Then the above says that with high probability,
\begin{equation}
    \frac{\abs{M_n}}{\sqrt{[M]_n \lp \log \frac 1\rho + \log\log \frac{en}{[M]_n}\rp}} \lesssim 1
\end{equation}
One way of thinking about this is to consider it as a non-asymptotic version of the martingale Laws of the Iterated Logarithm proved in \citep{victor1999general,bercu2008exponential}.

Much like the classical case, structural results of the function class can help control the majorizing measure's upper bound.  One example of such a result concerns Lipschitz compositions:
\begin{proposition}\label{prop:lipschitzmajorizingmeasures}
    Let $\F\subseteq[0,1]^{\mathcal Z}$ and let $\ell: [0,1] \to [0,1]$ be $L$-Lipschitz.  Then $I_{\ell \circ \F, \z}^\alpha \leq L I_{\F, \z}^\alpha$.
\end{proposition}

While majorizing measures in the classical case fully characterize the associated Gaussian process, constructing these multi-scale measures for a given class of functions at hand is notoriously difficult.  In the following section, inspired by \citep{alon2021Adversarial} for the binary-valued case, we introduce a single-scale measure that is easier to analyze than the multi-scale construction.

\section{Fractional Covers}

As we just remarked, the multi-scale nature of the majorizing measure makes it difficult to apply the technique to practical function classes.  Thus we introduce the fractional covering number for real-valued function classes as a single-scale alternative:
\begin{definition}\label{def:fractionalcover}
  Let $\z$ a binary, $\mathcal{Z}$-valued tree of depth $n$ and let $\F \subset \rr^{\mathcal{Z}}$.  A measure $\mu$ on the space of binary real-valued trees of depth $n$ is a fractional covering at scale $\delta$ of size $\gamma$ if for all $f \in \F$ and $\epsilon \in \{\pm1\}^n$, we have
  \begin{equation}
      \mu\lp \left\{ \vbf : \norm{\vbf(\epsilon) - f(\z(\epsilon))}^2 \leq n \delta^2 \right\}\rp \geq \frac 1\gamma.
  \end{equation}
  The fractional covering number $N'(\F, \delta, \z)$ of $\F$ at scale $\delta$ is the smallest $\gamma$ such that there exists a fractional cover of $\F$ at scale $\delta$ of size $\gamma$.  Let $N'(\F, \delta) = \sup_\z N'(\F, \delta, \z)$.
\end{definition}
We see immediately (and prove in Appendix \ref{app:miscellany}) that this new notion of size is controlled by the sequential covering number introduced in \citep{rakhlin2010online}:
\begin{lemma}\label{lem:fractionalvsclassicalcover}
  For any $\z, \F, \delta$, we have $N'(\F, \delta, \z) \leq N(\F, \delta, \z)$.
\end{lemma}
Interestingly, at least in the classical case, the reverse inequality holds as well:
\begin{lemma}\label{lem:classicalfractionalcover}
  Let $\z$ be a tree whose labels are constant as a function of depth.  Then $N\lp \F, 2\delta, \z\rp \leq N'(\F, \delta, \z) \leq N(\F, \delta, \z)$.
\end{lemma}
The proof of Lemma \ref{lem:classicalfractionalcover} is deferred to Appendix \ref{app:miscellany}.  Thus, at least in the classical case, the fractional covering number does not provide any improvement over the classical notion of covering number; whether such a result holds in the sequential case remains open, due to the lack of packing-covering duality in this setting.

To understand the connection between fractional covers and majorizing measures, we note that Definition \ref{def:fractionalcover} can be rephrased to say that $\mu_\delta$ is a $\delta$-fractional cover of size $\gamma$ if
\begin{equation}\label{eq:fractionalcovercomparison}
    \sup_{\substack{f \in \F \\ \epsilon \in \{\pm 1\}^n}} \log \frac 1{\mu_\delta\lp B_\delta(f(\z), \epsilon)\rp} = \log \gamma.
\end{equation}
Applying Theorem \ref{thm:majorizingchaining} and setting $\alpha = 0$ for ease of exposition, we see that, after putting the supremum inside of the integral,
\begin{align}\label{eq:sec31}
	\ee\left[\sup_{f \in \mathcal{F}} \mathbb{T}_f(\epsilon)\right] &\lesssim \sqrt{n}  \sup_{\substack{f \in \mathcal{F} \\ \epsilon \in \left\{\pm 1 \right\}^n}} \int_0^1 \sqrt{\log \frac 1{\mu(B_\delta(f(\z),\epsilon)}} d \delta \\
	&\leq  \sqrt n  \int_0^1 \sup_{\substack{f \in \mathcal{F} \\ \epsilon \in \left\{\pm 1 \right\}^n}}\sqrt{\log \frac 1{\mu(B_\delta(f(\z), \epsilon))}} d \delta
\end{align}
for all measures $\mu$.  The last integrand above \emph{almost} looks like the left-hand side of \eqref{eq:fractionalcovercomparison} and so we might hope that we can replace this integrand with $\sqrt{\log N'(\F, \delta, \z)}$.  The problem with this last step is that the measure $\mu_\delta$ of the fractional cover is allowed to depend on the scale while the majorizing measure is not; thus it is not obvious that a chaining bound with fractional covers follows from Theorem \ref{thm:majorizingchaining}. Fortunately, as the following proposition demonstrates, we are still able to control the majorizing measure by fractional covers and thus the above critique does not apply:
\begin{proposition}\label{prop:majorizingmeasure}
	Let $\z$ be a $\mathcal{Z}$-valued  binary tree of depth $n$ and let $\F$ be a class of real-valued functions on $\mathcal{Z}$.  Then
	\begin{equation}
		I_{\F, \z}^\alpha ~\lesssim~  \alpha +  \frac 1{\sqrt n} \int_\alpha^1 \sqrt{\log N'(\F, \delta, \z)} d \delta
	\end{equation}
	In other words, fractional covers dominate majorizing measures, up to a constant.
\end{proposition}
A proof of Proposition \ref{prop:majorizingmeasure} can be found in Appendix \ref{app:miscellany}.  An immediate corollary of the above, which follows from combining Theorem \ref{thm:majorizingchaining} and Proposition \ref{prop:majorizingmeasure} is
\begin{corollary}\label{thm:chaining}
	Let $\mathcal{F}$ be a separable $[0,1]$-valued function class on $\mathcal{Z}$.  Let $\z$ be a $\mathcal{Z}$-valued binary tree of depth $n$.  Let $N'(\mathcal{F}, \delta)$ be the fractional covering number at scale $\delta$.  Then with probability at least $1 - \rho$ over $\epsilon$, for all $\alpha \geq 0$,
	\begin{equation}\label{eq:integral}
		\sup_{f \in \F}\lp \frac 1n \mathbb{T}_{f(\z)}(\epsilon)\rp ~\lesssim~ \alpha +   \frac 1{\sqrt n} \int_\alpha^1 \sqrt{\log N'(\F, \delta, \z)} d \delta  + \sqrt{\frac{\log \frac 1\rho}n} .
	\end{equation}
	In particular, the sequential Rademacher complexity is bounded by the Dudley integral:
	\begin{equation}
		\R_n(\F, \z) ~\lesssim~ \inf_{\alpha > 0}\lp \alpha + \frac 1{\sqrt{n}} \int_\alpha^1 \sqrt{\log N'(\mathcal{F}, \delta, \z)} d \delta \rp .
	\end{equation}
\end{corollary}
If we apply Lemma \ref{lem:fractionalvsclassicalcover} to Corollary \ref{thm:chaining}, we recover the chaining bound with respect to sequential non-fractional covers of \citep{rakhlin2015sequential}.  As a first application of our results, using the techniques of \cite{rakhlin2015sequential} we get a high-probability uniform deviations bound:
\begin{corollary}\label{cor:uniformconcentration}
    Let $Z_1, \dots, Z_t, \dots$ be a sequence of $\mathcal{Z}$-valued random variables adapted to a filtration $\mathcal{A}_t$ and let $\F$ be a $[0,1]$-valued function class on $\mathcal{Z}$.  Then with probability at least $1 - 4 \rho$,
    \begin{align}
        \sup_{f \in \F} \abs{\frac 1n \sum_{t = 1}^n f(Z_t) - \ee[f(Z_t)| \mathcal{A}_t]} &~\lesssim~ \inf_{\alpha > 0} I_{\F}^\alpha + \sqrt{\frac{\log \frac 1\rho}{n}} \\
        &~\lesssim~ \inf_{\alpha > 0}\lp  \alpha +  \frac{1}{\sqrt n} \int_\alpha^1 \sqrt{\log N'(\F, \delta)} d \delta \rp + \sqrt{\frac{\log \frac 1\rho}{n}} 
    \end{align}
\end{corollary}
Corollary \ref{cor:uniformconcentration} significantly improves on the results of \cite[Section 7]{rakhlin2015sequential} in two ways: first, we chain with respect majorizing measures and fractional covering numbers which are guaranteed to be at least as good as sequential covering numbers by Lemma \ref{lem:fractionalvsclassicalcover}; second, our bound removes an extraneous factor polynomial in $\log n$ and now matches the form of uniform concentration bounds from the classical regime.  In the following section, we show how combinatorial properties of $\F$ can help us bound fractional covering numbers.

\section{Fractional Covering and Fat-Shattering Dimension}

Thus far, we have been focused on somewhat abstract results, providing upper bounds on stochastic processes in terms of majorizing measures and fractional covering numbers without any evidence that these new quantities improve on the sequential covering numbers introduced in \cite{rakhlin2010online}; here, we rectify that and focus on how to control the fractional covering numbers with combinatorial and structural properties of the function class.

One of the cornerstones in classical learning theory was the development of Vapnik-Chernovenkis theory and how it relates to uniform convergence over a binary-valued function class \citep{vapnik1981necessary}.  This theory was then extended to real-valued functions with the development of fat-shattering or scale-sensitive dimension \citep{bartlett1996fat,kearns1994efficient}.  In the online world, \cite{littlestone1988learning} introduced an analogue of the VC dimension and \cite{rakhlin2015sequential} extended this definition to a sequential notion of the fat-shattering dimension.  Here, we recall the definition of the fat-shattering dimension and then use it to bound the fractional covering numbers independently of $n$ in a sequential analogue of the results of \cite{mendelson2003entropy}.  We then show that, when integrated in the chaining upper bound of Corollary \ref{thm:chaining}, we can further improve the bound and provide a sequential analogue of the results of \cite{rudelson2006combinatorics}, which is tight in the Donsker case.  Finally, we apply these results to prove a tight (up to constants) Lipschitz contraction bound on worst-case Rademacher complexity, improving on the structural results of \cite{rakhlin2015sequential}.

We begin by recalling the fat-shattering dimension from \cite[Definition 7]{rakhlin2015sequential}:
\begin{definition}
    Let $\z$ a $\mathcal{Z}$-labelled binary tree of depth $n$ and let $\F$ a real-valued function class.  We say that $\F$ is shattered by $\z$ at scale $\delta$ if there exists a real-valued binary tree $\mathbf{s}$ such that for all $\epsilon \in \{\pm 1\}^n$, there exists an $f \in \F$ such that for all $1 \leq t \leq n$,
    \begin{equation}
        \epsilon_t \lp f(\z_t(\epsilon)) - \mathbf{s}_t(\epsilon)\rp \geq \frac \delta 2.
    \end{equation}
    The (sequential) fat-shattering dimension $\fat_\delta(\F)$ is the maximal $n$ such that there exists a tree $\z$ of depth $n$ that shatters $\F$ at scale $\delta$.
\end{definition}
A sequential version of the Sauer-Shelah lemma was also proven in \cite{rakhlin2015sequential} and this was used to bound the $\ell^\infty$ sequential covering numbers by the fat shattering dimension.  In particular, \cite[Corollary 1]{rakhlin2015sequential} tells us that if $\F$ takes values in $[0,1]$ then for any tree $\z$ of depth $n$,
\begin{equation}
    N(\F,\delta,\z) \leq N_\infty(\F, \delta,\z) \leq \lp \frac{2 e n}{\delta}\rp^{\fat_\delta(\F)}
\end{equation}
where $N_\infty$ denotes the covering number with respect to $\ell^\infty$.  As noted \emph{in situ}, this method cannot provide $n$-independent covering number bounds for the $\ell^2$ regime because the $\ell^\infty$ covering numbers cannot be independent of dimension in general.  In the classical setting, Dudley extraction, pioneered in \cite{dudley1973}, allows for dimension-independent bounds, but, due to the lack of covering-packing duality in the sequential case, this method cannot be applied in the online regime.  A major advantage of the fractional covering number is that we are able to eliminate the dependence on $n$ and control the \emph{fractional} covering numbers by the fat-shattering dimension.  We have the following bound:
\begin{theorem}\label{thm:fatshattering}
	Let $\F: \mathcal{Z} \to [0,1]$ be a function class.  There exist universal constants $C, c$ such that for all $\delta > 0$, and all trees $\z$
	\begin{equation}
		N'(\F, \delta, \z) \leq \left(\frac{C }{\delta}\right)^{3\fat_{c\delta}(\F)}
	\end{equation}
\end{theorem}
The details of the proof of Theorem \ref{thm:fatshattering} can be found in Appendix \ref{app:fatshattering}, but we provide a sketch here.

We follow \cite{rakhlin2015sequential} and discretize the unit interval into sub-intervals of length proportional to $\delta$; this discretization transforms the real-valued function class $\F$ into one that takes on finitely many values.  Using this transformation, it suffices to bound the fractional covering number of a simpler class of functions $\F: \mathcal{Z} \to \{0,1,\dots, m\}$ by $\fat_2(\F)$, i.e., it suffices to prove
\begin{proposition}\label{prop:discretefatshattering}
	Suppose that $\mathcal{F} : \mathcal{Z} \to \mathcal{Y} =  \{0,1,\dots,m\}$ and that $\fat_2(\mathcal{F}) = d$.  Then there is a universal constant $C$ such that $N'\lp\mathcal{F},\frac 23 , \z \rp \leq \left(C m\right)^{3d}$.
\end{proposition}
We provide full details in Appendix \ref{app:fatshattering}, but provide a sketch here:
\begin{proof}[Sketch]
    We fix the tree $\z$ and construct the fractional cover recursively.  The idea of the construction is to partition the function class $\F$ into sub-classes $\F_j$ taking the value $j$ on the root.  Lemma \ref{lem:inductive-fat}, proved in Appendix \ref{app:fatshattering}, tells us that there are at most two such sub-classes such that $\fat_2(\F_j) = \fat_2(\F_{j'}) = \fat_2(\F)$ and that $\abs{j - j'} \leq 1$.  Let $j^*$ be the minimal $j$ such that $\fat_2(\F_j) = \fat_2(\F)$ and notice that $j^*$ and $j^*+1$ are potentially the only values of $j$ such that $\fat_2(\F_j) = \fat_2(\F)$.  We construct the fractional cover $\mu$ recursively as the following mixture: with probability $1-p$, where $p$ is a small number to be specified, we label the root as $j^*+1/2$ and label the two subtrees of the root independently, using the fractional cover with respect to the complete class $\F$. Otherwise, with probability $p$, we draw a uniformly random $j \in \{0,\dots,m\}\setminus \{j^*,j^*+1\}$, label the root as $j$ and each of the two subtrees independently from the fractional covers with respect to the subclass $\F_j$, which satisfies $\fat_2(\F_j)\le \fat_2(\F)-1$.
    An induction proof relying on the recursive construction, found in Appendix \ref{app:fatshattering} as Lemma \ref{lem:fatshattering}, then tells us that if $\fat_2(\F) \leq d$ and $\mu$ is as constructed above, then for any integer $k$,
    \begin{equation}
        \inf_{\substack{f \in \F \\ \epsilon \in \{ \pm 1\}^n}} \mu\lp \left\{\vbf | \sum_{t = 0}^n \mathbf{1}_{\abs{\vbf_t(\epsilon) - f(\z(_t\epsilon))} \geq 1} \leq k \right\}\rp \geq (1 - p)^{n+1+k} \lp \frac{p}{m} \rp^d \binom{k + d}{d}\enspace.
    \end{equation}
    Setting $k = \delta n$, $p = \frac dn$, the right hand side is at least $(c\delta/m)^d$.  By the construction, if $\abs{\vbf_t(\epsilon) - f(\z_t(\epsilon))} < 1$ then it is at most $\frac 12$ and otherwise, $\abs{\vbf_t(\epsilon) - f(\z_t(\epsilon))}\le m$.  Thus if $\vbf$ is in the above set, then $\norm{\vbf(\epsilon) - f(\z(\epsilon))}^2 \leq m^2 \delta n + \frac 14 (1 - \delta)n$.  Thus if $\delta \leq \frac{7}{36 m^2}$, the result follows by the definition of a fractional cover.
\end{proof}
Note that Proposition \ref{prop:discretefatshattering} is already a major improvement over \cite[Corollary 1]{rakhlin2015sequential}, the analogue for non-fractional sequential covers; in addition, our result matches the form of the classical counterpart, proven in \cite{mendelson2003entropy}.  Interestingly, the proof of Proposition \ref{prop:discretefatshattering} relies on the independence of sub-trees inherent to the online case; thus, the same proof does \emph{not} apply to the classical analogue, proven by use of Dudley extraction.  The key lemma, Lemma \ref{lem:inductive-fat}, very much takes advantage of the sequential setting and thus the Proposition \ref{prop:discretefatshattering} and consequently Theorem \ref{thm:fatshattering}, provide important examples of results that are proved more easily in the sequential case than in the classical setting.  An interesting future direction would be to extend the proof method to the classical regime and reprove the analogue of Theorem \ref{thm:fatshattering} without the use of Dudley extraction.

If we plug the result of Theorem \ref{thm:fatshattering} into that of Corollary \ref{thm:chaining}, we are able to bound the sequential Rademacher complexity as follows:
\begin{equation}\label{eq:naivechaining}
    \R_n(\F) = \sup_{\mathsf{depth}(\z) = n} \R_n(\F, \z)  ~\lesssim~ \frac 1{\sqrt n} \int_0^1 \sqrt{\fat_{\delta}(\F) \log \frac{1}{\delta}} d \delta
\end{equation}
where we take $\alpha = 0$ in Corollary \ref{thm:chaining} for the sake of simplicity of exposition.  Equation \eqref{eq:naivechaining} improves the results of \cite{rakhlin2015sequential} in that it shaves off a factor polynomial in $\log n$, but we might hope for an even better bound that removes the $\log \frac 1\delta$ factor.  In the classical regime, it was shown in \cite{rudelson2006combinatorics} that this is indeed possible and that
\begin{equation}\label{eq:classicalrademacher}
    \R_n(\F) ~\lesssim~ \frac 1{\sqrt n} \int_0^1 \sqrt{\fat_\delta(\F)} d \delta
\end{equation}
where just in \eqref{eq:classicalrademacher}, the Rademacher complexity and fat shattering dimension are the classical notions rather than the sequential notions.  In fact, we do get a sequential analogue to \eqref{eq:classicalrademacher} and the proof is much simpler than that in the classical regime.  We have:
\begin{proposition}\label{prop:fat-integral}
	Let $\F$ be a class of functions from $\mathcal{Z}$ to $[0,1]$.  Then for all $b > 0$,
    \begin{equation}
        \int_{b}^1 \sqrt{\log N'(\F,\delta)}d\delta ~\lesssim~ \int_{cb}^1 \sqrt{\fat_\delta(\F)} d\delta.
    \end{equation}
\end{proposition}
As main step in this proof, we bound the fractional covering numbers with respect to multiple fat numbers:
\begin{proposition}\label{prop:fatshattering-gen}
	There exist universal constants $c,C>0$ such that the following holds:
	Let $\mathcal{F}$ be a class of functions from $\mathcal{Z}$ to $\mathcal{Y} = [0,1]$.  Then for all $\delta > 0$,
	\begin{equation}\label{eq:multiscale-bnd}
		N'(\mathcal{F}, \delta) \leq \prod_{i=1}^{\infty} C^{i\cdot\fat_{c\cdot 2^i\delta}(\F)}
	\end{equation}
\end{proposition}
\begin{proof}[sketch]
Similarly to the proof of Theorem~\ref{thm:fatshattering}, by discretizing $[0,1]$ into sub-intervals of length $\approx \delta$, it suffices for us to prove the following: if $\F$ is a $\{0,\dots,m\}$-valued function class, then
\[
N'(\mathcal{F}, 2) \leq \prod_{i=1}^{\log m} C^{i\fat_{2^i}(\F)}.
\]
To prove the above statement, we construct a fractional cover $\mu$ recursively, in the same spirit as the proof of Proposition~\ref{prop:discretefatshattering}. Partition $\F$ into sub-classes $\F_j$ taking the value $j$ on the root, and Lemma~\ref{lem:inductive-fat} tells us that there are at most $2^i$ classes $j$ with $\fat_{2^i}(\F_j)=\fat_{2^i}(\F)$ and any two such classes $j,j'$ satisfy $|j-j'|< 2^i$. Define by $i(j)$ as the maximal integer $i$ such that $\fat_{2^i}(\F_j) <\fat_{2^i}(\F)$, and let $j^*$ be any minimizer of $i(j)$. Intuitively, $\F_{j^*}$ has the same high-scale fat numbers as $\F$. The fractional cover is defined as the mixture that labels the root as $j^*$ with probability $1-p$ for a small $0<p<1$ and for any other $j$, with probability $p\cdot 4^{-i(j)-1}$. The remainder of the tree is recursively drawn, as in Proposition~\ref{prop:discretefatshattering}. By an inductive formula, we derive that for any numbers $k_1,k_2,\dots,k_{\log m}$ and any bounds $d_1,\dots,d_{\log m}$ on the fat numbers $\fat_{2^1}(\F),\dots,\fat_{2^{\log m}}(\F)$, the fractional cover $\mu$ satisfies
\begin{multline}
\inf_{\substack{f \in \F \\ \epsilon \in \{ \pm 1\}^n}}
\mu\lp \left\{
	    \vbf \colon \forall i=1,\dots,\log_2 m, \sum_{t=0}^n \mathbf{1}_{|\vbf_t(\epsilon)-f(\z_t(\epsilon))|\ge 2^i}|\leq \sum_{i'=i}^{\log_2 m} k_{i'}
	\right\}
\rp \\
\geq (1 - p)^{n + 1} \prod_{i=1}^{\log_2 m} (1-p)^{k_i} \left(\frac{p}{4^{i+1}}\right)^{d_i} \binom{k_i+d_i}{d_i}.
\end{multline}
Substituting $k_i = 8^{-i}n$, we obtain that any $\vbf$ in the above set satisfies $\|\vbf(\epsilon)-f(\z(\epsilon))\| \le O(1)$. Bounding the binomial coefficients similarly to Proposition~\ref{prop:discretefatshattering}, the proof concludes, as we derive a lower bound on the fraction of trees close to $f(\z)$ on any path $\epsilon$.
\end{proof}
While admittedly still somewhat technical, we observe that the proof of Proposition \ref{prop:fat-integral}, appearing in Appendix~\ref{app:fat-integral}, is actually significantly easier than that required for the analogous result in the classical regime, appearing in \citep{rudelson2006combinatorics}.

Miraculously, for function classes $\F$ that are not too complex, Proposition \ref{prop:fat-integral} gives a tight bound on the worst-case sequential Rademacher complexity, up to constants.  In order to state such a result, we need to define what it means for a function class to be simple.
\begin{definition}\label{def:regular}
  For a constant $c > 0$ and $0 \leq p < 2$, we say that a function class is $(c, p)$-bounded if for all $0 < \delta < 1$,
  \begin{equation}
      \fat_\delta(\F) \leq c \delta^{- p}.
  \end{equation}
\end{definition}

We note that many practical function classes are $(c, p)$-bounded; for example, Littlestone classes satisfy the above for $c = \mathsf{Ldim}(\F)$ and $p = 0$.  We have the following corollary:
\begin{corollary}\label{cor:rademacher}
    Let $\F$ be a $(c, p)$-bounded class of functions.  Then
    \begin{equation}\label{eq:tightbound}
        \int_0^1 \sqrt{\fat_\delta(\F)} d \delta ~\lesssim~ \sqrt n \R_n(\F) ~\lesssim~ \int_0^1 \sqrt{\fat_\delta(\F)} d \delta
    \end{equation}
    where the constants depend on $\F$ only through $c,p$.  The upper bound holds with no restriction on the complexity of the class $\F$.
\end{corollary}

Thus for many classes of interest, we have a tight characterization of sequential Rademacher complexity, matching the characterization in the offline case.  A challenging open problem remains, however: in the classical setting, majorizing measures provide both an upper and lower bound on the supremum of a Gaussian process.  While our results can be extended to give an upper bound for the sequential Gaussian complexity, where the path is determined by the signs of independent Gaussian multipliers, the lower bound remains to be shown.

As an application of our results, we prove a tight Lipschitz contraction statement for worst-case Rademacher complexity in the $(c,p)$-bounded regime that is dimension-independent, the first such statement in the online case.
\begin{corollary}\label{cor:contraction}
 Let $\F$ be a $(c,p)$-bounded class of functions.  If $\ell$ is an $L$-Lipschitz function, we have
    \begin{equation}
        \R_n\lp \ell \circ \F \rp ~\lesssim~ L \R_n(\F)
    \end{equation}
    where, again, the constant depends on $\F$ only through $c$ and $p$.
\end{corollary}

Importantly, the statement of Corollary \ref{cor:contraction} is one about \emph{worst-case} Rademacher complexity as opposed to \emph{tree-dependent} Rademacher complexity.  A similar statement with a factor polynomial in $\log n$ was proved as \cite[Lemma 7]{rakhlin2015sequential}.  Whether an analogous result holds in the tree-dependent case remains a challenging open question.

\section{Sample Complexity of Online Learning}

We finish this paper with a discussion of online learning, showing that our techniques lead to sharp estimates on minimax regret. 

The online supervised learning problem is defined as follows. We fix a class of functions $\cF\subseteq [-1,1]^\cX$, known to the learner. The online prediction problem proceeds over $n$ rounds. On round $t$, the learner observes $x_t\in\cX$, makes a prediction $\pred_t\in\reals$, and observes the outcome $y_t\in\cY\subseteq\reals$. Predictions may be randomized, in which case the learner selects a distribution $q_t$ on round $t$ and draws $\pred_t\sim q_t$. Regret is defined as 
\begin{align}
    \mathsf{Reg}_n(\cF) = \En\left[\frac{1}{n}\sum_{t=1}^n \ell(\pred_t,y_t) - \inf_{f\in\cF} \frac{1}{n}\sum_{t=1}^n \ell(f(x_t),y_t) \right].
\end{align}
In turn, the minimax regret is defined as 
\begin{align}
    V_n(\cF) = \min_{\tau} \max_{\pi} ~\mathsf{Reg}_n (\cF)
\end{align}
where minimum is taken over all randomized strategies of the learner, and maximum is over all strategies of Nature for selecting the sequence (potentially adaptively and adversarially). While this minimax object appears to be complicated, by writing down a repeated sequence of minima and maxima per time step, it is possible to upper bound the minimax regret by a supremum of a collection of martingales indexed by $\cF$ \citep{abernethy2009stochastic, rakhlin2010online}. In particular, it is shown in \citep{rakhlin2010online} that for binary classification with $\cY=\{0,1\}$, $\cF\subseteq\cY^\cX$, and $\ell(a,b)=\mathbb{I}\{a\neq b\}$,
\begin{align}
\label{eq:minimax_regret_sandwiched}
    \R_n(\cF) \leq V_n(\cF) \leq 2\R(\cF)
\end{align}
where $\R_n(\cF)=\sup_{\bx} \R_n(\cF, \bx)$. The tight characterization \eqref{eq:minimax_regret_sandwiched} also holds for prediction with absolute value loss ($\ell(a,b)=|a-b|$, $\cY=[-1,1],$ and $\cF\subseteq \cY^\cX$), as well as for prediction with linear loss ($\ell(a,b)=-ab$, $\cY=[-1,1], \cF\subseteq \cY^\cX$) \citep{rakhlin2010online}.

For the online learning problems mentioned above, we conclude that
\begin{align}
\label{eq:value_fat}
    V_n(\cF) ~\lesssim~ \inf_{\alpha\geq 0} \left\{ \alpha + \frac{1}{\sqrt{n}}\int_\alpha^1 \sqrt{\fat_\delta(\F)} d \delta\right\}, 
\end{align}
improving on the corresponding estimates in \citep{rakhlin2010online}. In particular, from Corollary~\ref{cor:rademacher}, for classes satisfying Definition~\ref{def:regular}, the minimax regret 
\begin{align}
    V_n(\cF) ~\asymp~ \frac{1}{\sqrt{n}}\int_0^1 \sqrt{\fat_\delta(\F)} d \delta ~\asymp~ \sqrt{\frac{c}{n}}.
\end{align}

In particular, for binary 
classification, this recovers the sharp result of \citep{alon2021Adversarial}, as $\fat_\delta(\F)$ reduces to the Littlestone dimension in this case. More generally, \eqref{eq:value_fat} provides a succinct and sharp characterization of optimal regret in terms of combinatorial parameters of the class. Using the techniques developed in this paper, we can sharpen and remove extraneous logarithmic factors in a number of bounds in \citep{rakhlin2010online}. We mention that for online regression with square loss, optimal regret is characterized in terms of a localized notion of sequential Rademacher complexity (offset complexity), and the study of its behavior in terms of the notions introduced in this paper is left for future work.

\section*{Acknowledgements}

AB acknowledges support from the National Science Foundation Graduate Research Fellowship under Grant No. 1122374. AR acknowledges support from the ONR through award \#N00014-20-1-2336.

\bibliographystyle{authordate1}
\bibliography{FractionalCover}

\appendix

\section{Proof of Theorem \ref{thm:majorizingchaining}}\label{app:chaining}
We use the technique of chaining. 
In order to apply it, we need an adaptive tail bound for the tree process.  We have
\begin{lemma}\label{lem:adaptivetail}
	Let $\vbf$ be a labelled binary tree and let $\epsilon_1, \dots, \epsilon_n$ be independent Rademacher random variables.  Then for any $y > 0$, with probability at least $1 - \rho$, we have
	\begin{equation}
		\sum_{t = 1}^n \vbf_t(\epsilon)^2 \leq y
	\end{equation}
	implies that
	\begin{equation}
		\abs{\sum_{t = 1}^n \epsilon_t \vbf_t(\epsilon)} \leq 2\sqrt{2y\log \frac 1\rho}
	\end{equation}
\end{lemma}
\begin{proof}
	Let
	\begin{equation}
		M_s = \sum_{t \leq s} \epsilon_t \vbf_t(\epsilon)
	\end{equation}
	Note that $M_s$ is a martingale with conditionally symmetric increments because the $\epsilon_t$ are independent and the $\vbf_t$ are adapted.  Note that the quadratic variation is given by $[M_s] = \sum_{t = 1}^n \vbf_t(\epsilon)^2$.  Thus by de la Pena's inequality (see \cite{victor1999general,bercu2008exponential}), we have for any $x, y > 0$,
	\begin{equation}
		\pp\left(\abs{M_T} \geq x, \,\, [M_T] \leq y \right) \leq 2e^{- \frac{x^2}{2 y}}
	\end{equation}
	Setting $x = 2 \sqrt{2 y \log \frac 1\rho}$ concludes the proof.
\end{proof}
Using Lemma \ref{lem:adaptivetail}, we are able to prove Lemma \ref{lem:newmartingaleconcentration}.
\begin{proof}[Proof of Lemma \ref{lem:newmartingaleconcentration}]
  Let $y_j = 2^{-j} n$ and let $\rho_j = \rho j^{-2}$.  Then for any fixed $j$, Lemma \ref{lem:adaptivetail} tells us that
  \begin{align}
      \pp\lp \abs{\T_\vbf} > C \sqrt{y_j \log \frac j\rho} \text{ and } \norm{\vbf}^2 \leq y_j\rp \leq \frac 6{\pi^2}\rho_j
  \end{align}
  for some $C$.  Summing $\frac 6{\pi^2} \rho_j$ gives $\rho$ and so a union bound tells us that with probability at least $1 - \rho$, for all $j$,
  \begin{align}
      \textbf{If} \quad \norm{\vbf} \leq \sqrt{n 2^{-j}} && \textbf{Then} \quad \abs{\T_\vbf} \leq C \sqrt{n 2^{-j} \log \frac j\rho}
  \end{align}
  On the event that $\norm{\vbf} = 0$ the result is trivial.  Thus, suppose otherwise and take $j = \lceil e\log \frac{n}{\norm{\vbf}^2}\rceil$; the above holds.  Plugging in, we see that the ``if" statement is valid for this $j$ and thus we have
  \begin{align}
      \abs{\T_\vbf} \leq C \sqrt{n 2^{-j} \log \frac j\rho} \leq C \sqrt{\norm{\vbf}^2 \log \frac{\log\frac{e n}{\norm{\vbf}^2}}{\rho}}
  \end{align}
  with probability at least $1 - \rho$.
\end{proof}
We will also need a lemma to relate probability over the majorizing measure to probability over the $\epsilon$.  We have:
\begin{lemma}\label{lem:fubini}
	Let $\mu$ a measure over binary trees $\vbf$ and consider indicator variables $U(\vbf, \epsilon)$ and let $\alpha,\rho \in (0,1)$. Suppose that for each $\vbf \in \supp \mu$, $\ee_\epsilon(U(\vbf, \epsilon)) \leq \alpha \rho$.  Then with probability at least $1 - \rho$ on the $\epsilon$,
	\begin{equation}
		\ee_\mu\left(U(\vbf,\epsilon)\right) \leq \alpha
	\end{equation}
\end{lemma}
\begin{proof}
	This follows immediately from Fubini's theorem and Markov's inequality.
\end{proof}

We are now ready to prove the main chaining bound with the fractional covering number.  For $\vbf,\vbf' \in \mathcal{V}$ binary trees, we fix
\begin{align}
	\mathbb{T}_{\vbf,\vbf'}(\epsilon) = \sum_{t = 1}^n \epsilon_t (\vbf_t(\epsilon) - \vbf_t'(\epsilon))
\end{align}
and we recall that the norm is
\begin{align}
	\norm{\vbf(\epsilon) - \vbf'(\epsilon)} = \sum_{t = 1}^n (\vbf_t(\epsilon) - \vbf_t'(\epsilon))^2
\end{align}
With this notation fixed, we can prove the theorem.

We prove the following lemma: 
\begin{lemma}\label{lem:improved-multiscale}
Let $\mu$ be a probability distribution over $[0,1]$ valued trees of depth $n$ and define $\nu = \mu \otimes \mu$. With probability $1-\rho$, the following holds:
uniformly, for any $\gamma \geq 2$,
\begin{align}
      &\mu\lp\left\{\vbf \colon \abs{\T_\vbf(\epsilon)} > C \sqrt{n  \log\frac{\gamma}{\rho}}\right\}\rp \leq \frac 1{\gamma} \\
      &\nu\lp\left\{ (\vbf,\vbf')\colon \abs{\T_{\vbf,\vbf'}(\epsilon)} > C \norm{\vbf(\epsilon) - \vbf'(\epsilon)} \sqrt{\log \frac{\gamma}{\rho} + \log\log \frac{3\sqrt{n}}{\norm{\vbf(\epsilon) - \vbf'(\epsilon)}}}
      \right\} \rp 
      \leq \frac 1{\gamma}
  \end{align}
\end{lemma}
\begin{proof}
  We first note that it suffices to prove the result uniformly on $\gamma_j = 2^j$.  For any $\gamma$, let $\gamma_j$ maximal power of $2$ that is at most $\gamma$.  Then the sets in the statement of the lemma for $\gamma$ are contained in the corresponding sets for $\gamma_j$.  Moreover, $\gamma_j$ differs from $\gamma$ by at most a factor of $2$, that can be eaten into the constants $C$ in the statement.  More formally, we have
  \begin{align}
      \mu\lp\left\{\vbf \colon \abs{\T_\vbf(\epsilon)} > 2C \sqrt{n  \log\frac{\gamma}{\rho}}\right\}\rp &\leq\mu\lp\left\{\vbf \colon \abs{\T_\vbf(\epsilon)} > C \sqrt{n  \log\frac{\gamma}{2\rho}}\right\}\rp  \\
      &\leq \mu\lp\left\{\vbf \colon \abs{\T_\vbf(\epsilon)} > C \sqrt{n  \log\frac{\gamma_j}{2\rho}}\right\}\rp \\
      &\leq \frac 2{\gamma_j} \leq \frac 1\gamma
  \end{align}
  and similarly for the other event; thus, it suffices to prove the result only for $\gamma$ powers of $2$.
  
  Now, setting $y = \sqrt n$ and applying Lemma \ref{lem:adaptivetail}, we get
  \begin{align}
      \pp\lp \left\{\abs{\T_\vbf(\epsilon)} > C \sqrt{n  \log\frac{\gamma_j^2}{\rho}}\right\}\rp &= \pp\lp \left\{\abs{\T_\vbf(\epsilon)} > C \sqrt{n  \log\frac{\gamma_j^2}{\rho}} \text{ and } \norm{\vbf(\epsilon)} \leq \sqrt n \right\}\rp \\
      &\leq \frac \rho{6\gamma_j^2}
  \end{align}
  for $C$ large enough, where the equality follows from the fact that $\abs{\vbf_t(\epsilon)} \leq 1$ and so $\norm{\vbf(\epsilon)} \leq \sqrt n$ almost surely.  Applying Lemma \ref{lem:fubini} tells us that for all $j$, with probability at least $1 - \frac \rho{6\gamma_j}$, we have
  \begin{equation}
      \mu\lp\left\{\vbf \colon \abs{\T_\vbf(\epsilon)} > C \sqrt{n  \log\frac{\gamma_j^2}{\rho}}\right\}\rp \leq \frac 1{\gamma_j} 
  \end{equation}
  Taking a union bound and noting that $\frac 1{\gamma_j} = 2^{-j}$ is summable, then pulling the square out of the log factor tells us that with probability at least $1 - \frac \rho6$, we have, uniformly in $j$,
  \begin{equation}
      \mu\lp\left\{\vbf \colon \abs{\T_\vbf(\epsilon)} > C \sqrt{n  \log\frac{\gamma_j}{\rho}}\right\}\rp \leq \frac 1{\gamma_j} 
  \end{equation}
  
  Now we look at the second event.  The identical argument serves to allow us to only consider $\gamma_j$.  We now apply Lemma \ref{lem:newmartingaleconcentration} to Lemma \ref{lem:fubini} instead of Lemma \ref{lem:adaptivetail}.  In particular, note that for any $\gamma_j$, and any $\vbf,\vbf'$, we have
  \begin{align}
      \pp\lp \abs{\T_{\vbf,\vbf'}(\epsilon)} > C \norm{\vbf(\epsilon) - \vbf'(\epsilon)} \sqrt{\log \frac{6\gamma_j^2}\rho + \log\log \frac{6\sqrt n}{\norm{\vbf(\epsilon) - \vbf'(\epsilon)}}}\rp \leq \frac{\rho}{6\gamma_j^2}
  \end{align}
  by Lemma \ref{lem:newmartingaleconcentration}.  Thus, applying Lemma \ref{lem:fubini} with the measure $\nu$, we have that with probability at least $1 - \frac{\rho}{\gamma_j}$,
  \begin{align}
      \nu\lp\left\{ (\vbf,\vbf')\colon \abs{\T_{\vbf,\vbf'}(\epsilon)} > C \norm{\vbf(\epsilon) - \vbf'(\epsilon)} \sqrt{\log \frac{\gamma_j^2}{\rho} + \log\log \frac{3\sqrt{n}}{\norm{\vbf(\epsilon) - \vbf'(\epsilon)}}}
      \right\} \rp 
      \leq \frac 1{\gamma_j}
  \end{align}
  Applying a union bound and taking out the square from the logarithm and putting it into the constant $C$ concludes the proof.
\end{proof}

For the remainder, assume that the event of Lemma \ref{lem:improved-multiscale} holds and we will show how to bound the supremum over $\T_{f(\z)}$.
Define $\delta_j = 2^{-j}$ and for any $f,\epsilon$, define $\gamma_j(f,\epsilon) = 1/\mu(B_{\delta_j}(f,\epsilon))$. Let $N$ be the maximal integer such that $\delta^N \ge 2\alpha$.

Notice that it suffices to prove that (under the event of Lemma~\ref{lem:improved-multiscale}), for any $f \in \F$,
\begin{align}
		\mathbb{T}_f(\epsilon) &\leq C\lp \sqrt{n \log \frac 1\rho} + n \delta_N  + \sqrt{n}\sum_{j = 0}^N  \delta_j\sqrt{\log \gamma_j(f,\epsilon)}\rp \\
		&\leq C \lp \sqrt{n \log \frac 1\rho} + n\alpha + \sqrt{n}\sum_{j = 0}^N  2^{-j} \sqrt{\log\frac{1}{\mu(B_{2^{-j}}(f,\epsilon))}}\rp.
\end{align}

Define 
\begin{align}
r_j(f,\epsilon) = C \delta_j \sqrt{n} \sqrt{\log \frac{\gamma_{j+1}(f,\epsilon)}{\rho} + \log\log \frac{3}{\delta_j}}
\end{align}
for a sufficiently large $C>0$ and $j \ge 0$ integer. We write $r_j$ and $\gamma_j$ when $f,\epsilon$ are obvious from context.
We have that the following holds for any $f,\epsilon$ and $j \ge 1$ (we refer to these inequalities as \hypertarget{majorizingtarget2}{$(\dagger)$}):
\begin{align}
      &\mu\lp\left\{\vbf \colon \abs{\T_\vbf(\epsilon)} > r_0(f,\epsilon)\right\}\rp \leq \frac 1{4 \gamma_0(f,\epsilon)} \\
      &\nu\lp\left\{ (\vbf,\vbf')\colon \abs{\T_{\vbf,\vbf'}(\epsilon)} > r_j(f,\epsilon) \text{ and } \norm{\vbf(\epsilon) - \vbf(\epsilon')} \leq 2 \delta_j \sqrt n \right\} \rp \\
      &\leq \frac 1{9 \gamma_j(f,\epsilon) \gamma_{j+1}(f,\epsilon)}.
  \end{align}
  Indeed, the first inequality is obtained from Lemma~\ref{lem:improved-multiscale} by substituting $\gamma=4\gamma_0(f,\epsilon)$ and the second inequality is obtained by taking $\gamma = 9 \gamma_j(f,\epsilon) \gamma_{j+1}(f,\epsilon) \le 9 \gamma_{j+1}(f,\epsilon)^2$.
  
  Our goal is to construct a chain of elements $\vbf_{f,0},\dots,\vbf_{f,N}$ that provide finer and finer approximations for $f(\z)$ on the path $\epsilon$, and obtain the bound
  \[
  \T_{f(\z)} \le \T_{\vbf_{f,0}} + \sum_{j=0}^{N-1} \T_{\vbf_{f,j},\vbf_{f,j+1}}
  + \T_{\vbf_{f,N},f(\z)}.
  \]
  The approximations will be chosen from the sets $A'_j(f,\epsilon)$ defined below:
  \begin{align}
		A_j(f, \epsilon) &= \left\{\vbf| \norm{\vbf(\epsilon) - f(\z(\epsilon))} \leq \sqrt n \delta_j \right\} \\
		A_j'(f, \epsilon) &= \left\{\vbf \in A_j(f,\epsilon)\colon \mu\left(\left\{\vbf' | \abs{\mathbb{T}_{\vbf,\vbf'}(\epsilon)} > r_j(f,\epsilon)\text{ and } \norm{\vbf(\epsilon) - \vbf'(\epsilon)} \leq 2 \delta_j \sqrt{n} \right\}\right) \leq \frac 1{3 \gamma_{j+1}} \right\}
	\end{align}
	
	As a first step, we would like to lower bound $\mu(A_j'(f, \epsilon))$ assuming we have $\epsilon$ such that \hyperlink{majorizingtarget2}{($\dagger$)} holds.
	By construction, $\mu(A_j(f, \epsilon)) \geq \frac 1{\gamma_j(f,\epsilon)}$.  
	For any $\vbf$, let
	\begin{equation}
		B(\vbf) = \left\{\vbf' | \abs{\mathbb{T}_{\vbf,\vbf'}(\epsilon)} > r_j(f,\epsilon) \text{ and } \norm{\vbf(\epsilon) - \vbf'(\epsilon)} \leq 2 \delta_j \sqrt{n} \right\}
	\end{equation}
	and let
	\begin{equation}
		B = \bigcup_\vbf \bigcup_{\vbf' \in B(\vbf)} \left\{(\vbf,\vbf') \right\}
	\end{equation}
	Let
	\begin{equation}
		\widetilde{B} = \left\{\vbf | \mu(B(\vbf)) > \frac 1{3 \gamma_{j+1}} \right\}
	\end{equation}
	From \hyperlink{majorizingtarget2}{($\dagger$)}, we know that $\nu(B) \leq \frac 1{9 \gamma_j \gamma_{j+1}}$.  By construction of $\nu$, we have
	\begin{align}
		\nu(B) \ge \mu\left(B(\vbf) | \vbf \in \widetilde{B}\right)  \mu\left(\widetilde{B}\right)
	\end{align}
	By construction, if $\vbf \in \widetilde{B}$, then
	\begin{equation}
		\mu\left(B(\vbf)\right) \geq \frac 1{3 \gamma_{j+1}}
	\end{equation}
	Thus we have
	\begin{align}
		\mu \left(\widetilde{B}\right) \leq \frac{\nu(B)}{\mu\left(B(\vbf) | \vbf \in \widetilde{B}\right) } \leq \frac{\frac 1{9 \gamma_j \gamma_{j+1}}}{\frac 1{3 \gamma_{j+1}}} = \frac 1{3 \gamma_j}
	\end{align}
	Thus we see that
	\begin{equation}
		\mu(A_j'(f, \epsilon)) \geq \mu(A_j(f, \epsilon)) - \mu(\widetilde{B}) \geq \frac 1{\gamma_j} - \frac 1{3 \gamma_j} = \frac 2{3 \gamma_j}\enspace.
	\end{equation}
	As a final remark before we construct the chain, note that if $\vbf \in A_j(f, \epsilon)$ and $\vbf' \in A_{j+1}(f, \epsilon)$, then by the triangle inequality,
	\begin{equation}\label{eq:triangle}
		\norm{\vbf(\epsilon) - \vbf'(\epsilon)} \leq \norm{\vbf(\epsilon) - f(\z(\epsilon))} + \norm{\vbf'(\epsilon) - f(\z(\epsilon))} \leq \delta_j \sqrt{n} + \delta_{j+1} \sqrt{n} \leq 2 \delta_j \sqrt n
	\end{equation}
	We are now ready to construct the chain.  Fix $\epsilon$ such that \hyperlink{majorizingtarget2}{($\dagger$)} holds.  For notational simplicity, we suppress the dependence of $\vbf_{f,j,\epsilon}$ on $\epsilon$ and simply write $\vbf_{f,j}$; this causes no confusion as we have fixed $\epsilon$.  For any $f$, we have
	\begin{align}
		\mu\left(A_j'(f,\epsilon) \cap \left\{\vbf | \abs{\mathbb{T}_\vbf(\epsilon)} > r_j(f,\epsilon) \right\}^c\right) \geq \frac 2{3 \gamma_j} - \frac 1{3 \gamma_j} = \frac 1{3\gamma_j} > 0
	\end{align}
	Thus, letting $j = 0$, we see that there exists a $\vbf_{f,0} \in A_{0}'(f,\epsilon)$ such that
	\begin{equation}
		\abs{\mathbb{T}_{\vbf_{f,0}}(\epsilon)} \leq r_0(f,\epsilon).
	\end{equation}
	Now suppose that we have chosen $\vbf_{f,j} \in A_j'(f, \epsilon)$.  We wish to select an element $\vbf_{f, j+1} \in A_{j+1}'(f, \epsilon)$ such that $\abs{\mathbb{T}_{\vbf_{f,j}, \vbf_{f,j+1}}(\epsilon)}$ is small.  To see that this is possible, we remark that, as $\vbf_{f,j} \in A_j'(f, \epsilon)$, we have by definition
	\begin{align}
		\mu\left(\left\{\vbf' | \abs{\mathbb{T}_{\vbf_{f,j},\vbf'}(\epsilon)} > r_j(f,\epsilon) \text{ and } \norm{\vbf_{f,j}(\epsilon) - \vbf'(\epsilon)} \leq 2 \delta_j \sqrt{n} \right\}\right) \leq \frac 1{3 \gamma_{j+1}}
	\end{align}
	But we have already seen that
	\begin{align}
		\mu(A_{j+1}'(f, \epsilon)) \geq \frac 2{3 \gamma_{j+1}}
	\end{align}
	and as noted, for any $\vbf' \in A_{j+1}'(f,\epsilon)$, $\norm{\vbf_{f,j}(\epsilon) - \vbf'(\epsilon)} \leq 2 \delta_j \sqrt n$.  Thus $\mu$ takes measure at least $\frac 1{3 \gamma_{j+1}}$ on the set of such desirable next links in the chain and so such a $\vbf_{f,j+1}$ exists.  Thus for any $f$, on this $\epsilon$, we have a chain $\vbf_{f,j}$ for $j \geq 0$ such that $\vbf_{f, j} \in A_j'(f, \epsilon)$.  Note that, by \eqref{eq:triangle} and the way we have constructed the chain, we have that for all $f$,
	\begin{equation}
		\abs{\mathbb{T}_{\vbf_{f,j}, \vbf_{f,j+1}}(\epsilon)} \leq  r_j(f,\epsilon).
	\end{equation}
	We also have for any $f$ that
	\begin{equation}
		\abs{\mathbb{T}_{\vbf_{f,0}}} \leq r_0(f,\epsilon).
	\end{equation}
	Thus we see that on this $\epsilon$, for every $f$,
	\begin{align}
      \T_f(\epsilon) &= \T_{\vbf_{f,0,\epsilon}}(\epsilon) + \sum_{N \geq j \geq 0} \T_{\vbf_{f, j, \epsilon},\vbf_{f,j+1,\epsilon}}(\epsilon) + \T_{f(\z), \vbf_{f,N,\epsilon}}(\epsilon) \\
      &\le r_0(f,\epsilon) + \sum_{j=0}^{N-1} r_j(f,\epsilon) + Cn \delta_n(f,\epsilon)\\
      &\leq C \sqrt{n \log \frac 1 \rho} + \sum_{j=0}^{N-1} C \sqrt{n} \lp\sqrt{\log \gamma_j(f,\epsilon)} + \sqrt{\log\log \frac{3}{\delta_j}}\rp\delta_j + \delta_N(f, \epsilon) n \\
      &\leq C \sqrt{n \log \frac 1 \rho} + \sum_{j=0}^{N-1} C \sqrt{n} \sqrt{\log \gamma_j(f,\epsilon)} \delta_j + \delta_N(f, \epsilon) n
  \end{align}
  where the inequality $\T_{f(\z), \vbf_{f,N,\epsilon}}(\epsilon)$ follows from Cauchy-schwarz and the fact that $\vbf_{f,N} \in A_N(f, \epsilon)$.  As \hyperlink{events}{($\dagger$)} occurs with probability at least $1 - \rho$, we have with the same probability that the above is bounded.  This concludes the proof.

\section{Proof of Theorem \ref{thm:fatshattering}}\label{app:fatshattering}

In this appendix, we prove the auxiliary lemmata required to complete the proof of Theorem \ref{thm:fatshattering}.  We begin with a result from \cite{rakhlin2015sequential}:
\begin{lemma}[\cite{rakhlin2015sequential}]\label{lem:inductive-fat}
    Let $\z$ be a $\mathcal{Z}$-valued binary tree, $\F$ a class of functions $\mathcal{Z} \to \{0,1,\dots,m\}$, $\z_0$ the root of the tree, and $\F_j = \{f \in \F | f(\z_0)=j\}$.  Let $r\geq 1$ be an integer.  If $j,j' \in \{0, \dots, m\}$ such that $\fat_r(\F_j)=\fat_r(\F_{j'})=\fat_r(\F)$, then $|j-j'| < r$. In particular, there can be at most $r$ values of $j$ such that $\fat_r(\F_j)=\fat_r(\F)$.
\end{lemma}
\begin{proof}
    Suppose that $j,j' \in \{0, \dots, m\}$ such that $j\ge j'+r$. Let $\z^{(-1)}$ and $\z^{(1)}$ denote trees of depth $\fat_r(\F)$ that shatter at scale $r$ the classes $\F_j$ and $\F_{j'}$. Then, we can construct a tree of depth $\fat_r(\F)+1$ that $r$-shatters $\F$ by taking $z_0$ as a root and taking $z^{(\pm 1)}$ as subtrees of the root. This contradicts the fact that $\fat_r(\F)$ is the maximal depth of a tree that shatters $\F$ at scale $r$.  The result follows.
	\end{proof}
We now provide the detailes for the proof of Proposition \ref{prop:discretefatshattering}, sketched in the text.
\begin{proof}[Proof of Proposition \ref{prop:discretefatshattering}]
   Given a tree $\z$ and $\F$, we construct the fractional cover $\mu_{\z,\F}$ recursively.
   If $\z$ is an empty tree then define $\mu_{\z,\F}$ to have a unit mass on the empty tree. Otherwise, define by $\z_0$, $\z^{(-1)}$ and  $\z^{(1)}$ its root and left and right subtrees, respectively (if the tree has depth zero then $\z^{(\pm 1)}$ are emty trees). Further, define $\F_j = \left\{f \in \F: f(\z_0) = j \right\}$ for $0 \leq j \leq m$. We form $\mu$ as a mixture of measures $\mu_{z,\F_j,j}$ and $\mu_{\z,\F,j}$, where $\mu_{\z,\mathcal{G},j}$ is defined as the measure on $\mathcal{Y}$-valued trees $v$ such that the root is labeled by $i$ almost surely and the left and right subtrees are sampled independently from $\mu_{\z^{(-1)},\mathcal{G}}$ and $\mu_{\z^{(1)},\mathcal{G}}$.
   
   To define the mixture coefficients, we use Lemma~\ref{lem:inductive-fat} which tells us that if $j < j'$ are such that $\fat_2(\F_j) = \fat_2(\F_{j'}) = \fat_2(\F)$, then $j' = j + 1$. Let $j^*$ be the minimal $j$ such that $\fat_2(\F_j) = \fat_2(\F)$ (if such a $j$ exists, if not let $j^* = 0$) and notice that for $j \not\in \{j^*, j^* + 1 \}$, it hold that $\fat_2(\F_j) \leq \fat_2(\F) - 1$. 
   
   Fix $0 < p < 1$ to be determined later and construct $\mu_{\z,\F}$ by sampling $\mu_{\z,\F,j^*+1/2}$ with probability $1-p$, and sampling $\mu_{\z,\F_j,j}$ with probability $\frac{p}{m-1}$ for all other $j$.  We claim that with an appropriate choice of $p$, the above construction yields a $\frac 23$-fractional cover of the correct size.
	
    In order to prove the claim, for a depth $n$ binary tree $\z$, let
	\begin{equation}
	\phi(k, \mathcal{F}, \z) = \min_{\substack{f \in \mathcal{F} \\ \epsilon \in \{\pm1\}^n}} \mu\left(\left\{\vbf \colon \sum_{t = 0}^n\mathbf{1}_{\left\{|\vbf_t(\epsilon) - f(\z_t(\epsilon))| \geq 1\right\}} \leq k \right\}\right)
	\end{equation}
	and let
	\begin{equation}
	\phi(k, d, n) = \min_{\substack{\fat_2(\mathcal{F}) \leq d \\ \text{depth of } \z \leq n}} \phi(k, \mathcal{F}, \z).
	\end{equation}
	
	We claim (and prove as Lemma \ref{lem:fatshattering} at the conclusion of this proof so as not to interrupt the argument) that
	\begin{equation}
	\phi(k, d, n) \geq (1 - p)^{n+1+k} \left(\frac{p}{m-1}\right)^d \binom{k + d}{d} \geq (1 - p)^{n+1+k} \left(\frac pm\right)^d \left(\frac{k}{d}\right)^d
	\end{equation}
	for all $0 < p < 1$, where the second inequality comes from the elementary lower bound $\binom{n}{k} \ge (n/k)^k$. Now, letting $k = \delta n$ for $\delta=\frac 7{36 m^2}<1$ and $p = \frac dn$, we note that for $n \ge 2d$,
	\begin{equation}
	\phi(k, d, n) \geq \left(1 - \frac dn\right)^{2n} \left(\frac d{mn} \frac{\delta n}{d}\right)^d \geq 
	c \left(\frac{e^2\delta}{dm}\right)^d.
	\end{equation}
    Note that as $\vbf_t(\epsilon), f(\z_t(\epsilon)) \in \left\{0 , \dots, m ,j^* + \frac 12\right\} $, if $\abs{\vbf_t(\epsilon) - f(\z_t(\epsilon))} < 1$, then it is at most $\frac 12$.  Thus, if
    \begin{equation}
        \sum_{t = 0}^n\mathbf{1}_{\left\{|\vbf_t(\epsilon) - f(\z_t(\epsilon))| \geq 1\right\}} \leq \delta n
    \end{equation}
	then
	\begin{equation}
	    \norm{\vbf_t(\epsilon) - f(\z_t(\epsilon))}^2 =  \sum_{t = 1}^n (\vbf_t(\epsilon) - f(\z_t(\epsilon)))^2 \leq m^2\delta n + \frac 12 (1 - \delta)n.
	\end{equation}
	Recall that $\delta = \frac 7{36 m^2}$ and this concludes the proof.
\end{proof}
We now prove the technical lemma used above:
\begin{lemma}\label{lem:fatshattering}
	Let $\phi(k,d,n)$ be as defined in the above proof.  Then for all $0 < p < 1$, and $n \geq -1$ (where $n=-1$ corresponds to an empty tree),
	\begin{equation}\label{eq:one-scale-bound}
	\phi(k,d,n) \geq (1 - p)^{n+1+k} \left(\frac{p}{m-1}\right)^d \binom{k+d}{d}
	\end{equation}
\end{lemma}
\begin{proof}
		We prove the result by induction on $n$. 
		If $n=-1$ then the left hand side of \eqref{eq:one-scale-bound} equals $1$, while the right hand side equals
		\begin{multline*}
		(1-p)^{k} \left(\frac{p}{m-1} \right)^{d} \binom{k+d}{d}
		\le (1-p)^{k} p^{d} \binom{k+d}{d}
		\le \sum_{\ell=0}^{k+d} (1-p)^{k+d-\ell} p^{\ell} \binom{k+d}{\ell}\\
		= ((1-p)+p)^{k + d}
		= 1.
		\end{multline*}

		Thus, we may assume that $n \ge 0$. We prove the following recursive formula:
		
		\begin{equation}\label{eq:phirecursion}
	    \phi(k,d,n) \geq
	    \min\begin{pmatrix}
	    (1-p)\phi(k,d,n-1), \\
	    \begin{cases}
	    (1 - p) \phi(k-1,d,n-1) + \frac{p}{m-1} \phi(k,d-1,n-1)
	        & \text{if } k,d>0\\
	    \frac{p}{m-1} \phi(k,d-1,n-1)
	        & \text{if } k=0,d>0\\
	    \infty & if d=0
	    \end{cases}
	    \end{pmatrix}
	    \end{equation}
	To see this, consider a fixed $f,\epsilon$.  We divide into cases:
	\paragraph{If $f(\z_0) \in \{j^*, j^* + 1\}$.}  
	Recall that $\mu_{\z,\F}$ samples $\mu_{\z,\F,j^*+1/2}$ with probability $1-p$. For any $\vbf\in \mathrm{support}(\mu_{\z,\F,j^*+1/2})$ it holds that $\vbf_0=j^*+1/2$ and so $\abs{\vbf_0 - f(\z_0)} \leq \frac 12 < 1$. Hence,
	\begin{multline*}
	    \mu_{\z,\F}\lp\left\{\vbf \colon \sum_{t = 0}^n\mathbf{1}_{\left\{|\vbf_t(\epsilon) - f(\z_t(\epsilon))| \geq 1\right\}} \leq k \right\}\rp \\
	    \geq (1 - p) \mu_{\z,\F,j^*+1/2}\lp\left\{\vbf \colon \sum_{t = 0}^n\mathbf{1}_{\left\{|\vbf_t(\epsilon) - f(\z_t(\epsilon))| \geq 1\right\}} \leq k \right\}\rp
	    \geq (1 - p) \phi(k, d, n-1).
	\end{multline*}
    \paragraph{If $f(\z_0) = j \not\in \{j^*, j^* + 1\}$ and $k,d>0$.}
    Here, differently from the previous case, for $\vbf\in \mathrm{support}(\mu_{\z,\F,j^*+1/2})$ it holds that $\abs{\vbf_0 - f(\z_0)} \geq 1$. Additionally, recall that $\mu_{\z,\F}$ samples $\mu_{\z,\F_j,j}$ with probability $\frac{p}{m-1}$. For any $\vbf \in \mathrm{support}(\mu_{\z,\F_j,j})$, it holds that $\vbf_0 = f(\z_0)$. Further, recall that $\fat_2(\F_j) \le \fat_2(\F)-1$. Hence, we see that in this case,
	\begin{multline*}
	    \mu_{\z,\F}\lp\left\{\vbf \colon \sum_{t = 0}^n\mathbf{1}_{\left\{|\vbf_t(\epsilon) - f(\z_t(\epsilon))| \geq 1\right\}} \leq k \right\}\rp 
	    \geq 
	    (1-p)\mu_{\z,\F,j^*+1/2}\lp\left\{\vbf \colon \sum_{t = 0}^n\mathbf{1}_{\left\{|\vbf_t(\epsilon) - f(\z_t(\epsilon))| \geq 1\right\}} \leq k \right\}\rp \\
	    + \frac{p}{m-1}\mu_{\z,\F_j,j}\lp\left\{\vbf | \sum_{t = 0}^n\mathbf{1}_{\left\{|\vbf_t(\epsilon) - f(\z_t(\epsilon))| \geq 1\right\}} \leq k \right\}\rp \\
	    \geq (1-p) \phi(k-1,d,n-1) + \frac{p}{m-1} \phi(k, d-1,n-1).
	\end{multline*}
    \paragraph{If $f(\z_0) = j \not\in \{j^*, j^* + 1\}$ and $k=0,d>0$.}
    Similarly to the previous case, we have
    \begin{multline*}
	    \mu_{\z,\F}\lp\left\{\vbf \colon \sum_{t = 0}^n\mathbf{1}_{\left\{|\vbf_t(\epsilon) - f(\z_t(\epsilon))| \geq 1\right\}} \leq k \right\}\rp 
	    \geq 
	    \frac{p}{m-1}\mu_{\z,\F_j,j}\lp\left\{\vbf | \sum_{t = 0}^n\mathbf{1}_{\left\{|\vbf_t(\epsilon) - f(\z_t(\epsilon))| \geq 1\right\}} \leq k \right\}\rp \\
	    \geq \frac{p}{m-1} \phi(k, d-1,n-1).
	\end{multline*}
    \paragraph{If $f(\z_0) = j \not\in \{j^*, j^* + 1\}$ and $d=0$.}
    This case cannot hold. Assume that it does towards contradiction. Since $f(\z_0) = j \not\in \{j^*, j^* + 1\}$, we have $\fat_2(\F_j)\le \fat_2(\F)-1=-1$, which derives the contradiction since the fat numbers are nonnegative.
    
    Thus we have shown \eqref{eq:phirecursion}. We divide into cases. If the minimum in the right hand side of \eqref{eq:phirecursion} equals $(1-p)\phi(k,d,n-1)$ then the claim trivially follows by induction hypothesis. Otherwise, we divide into cases according to $d,k$. If $k,d>0$ then
    \begin{multline*}
    \phi(k,d,n)
    \ge (1-p)\phi(k-1,d,n-1)+\frac{p}{m-1}\phi(k,d-1,n-1)\\
    \ge (1 - p)^{n+k+1} \left(\frac{p}{m-1}\right)^d \lp\binom{k+d-1}{d} + \binom{k+d-1}{d-1} \rp
    = (1 - p)^{n+k+1} \left(\frac{p}{m-1}\right)^d\binom{k+d}{d}.
    \end{multline*}
    If $k=0,d>0$, we have
    \begin{multline*}
    \phi(k,d,n)
    \ge \frac{p}{m-1}\phi(k,d-1,n-1)\\
    \ge (1 - p)^{n+1} \left(\frac{p}{m-1}\right)^d \binom{d-1}{d-1}
    = (1 - p)^{n+1} \left(\frac{p}{m-1}\right)^d\binom{d}{d},
    \end{multline*}
    as required. Lastly, the case $d=0$ cannot hold, as we assumed that the minimum in the right hand side of \eqref{eq:phirecursion} doe not equal $(1-p)\phi(k,d,n-1)$. This concludes the proof.
\end{proof}
Using Proposition \ref{prop:discretefatshattering}, we can bound the fractional covering number of a real valued function class, with respect to its fat-shattering dimension.  We introduce some notation.  For any $\alpha >0$ and function class $\F$ taking values in the unit interval, let $\lfloor\F\rfloor_\alpha$ be the $\frac \alpha 2$-discretization of $f$, i.e., for any $f \in \F$, let $\lfloor f \rfloor_\alpha = \frac\alpha 2 \left\lfloor \frac {2f}\alpha \right\rfloor$.  Let $\mathcal{G} = \frac 1{\alpha}\lfloor\F\rfloor_\alpha$.  We have the following lemma, reproven for the sake of completeness:
\begin{lemma}[\cite{rakhlin2015sequential}]\label{lem:rakhlindiscretization}
  With the notation as above, we have $\fat_2(\mathcal{G}) \leq \fat_{\alpha}(\F)$.
\end{lemma}
\begin{proof}
   Multiplying by $\alpha$, the statement is equivalent to $\fat_{2\alpha}(\lfloor\F\rfloor_\alpha) \leq \fat_\alpha(\F)$.  Let $\mathbf{w}$ be a tree that shatters $\lfloor\F\rfloor_\alpha$ at scale $\alpha$ of depth $r$.  By definition, there exists a tree $\mathbf{s}$ such that for all $\epsilon \in \{\pm1\}^r$, there is an $f \in \F$ such that for all $1 \leq t \leq r$
   \begin{equation}
       \epsilon_t\lp \lfloor f(\mathbf{w}_t(\epsilon)\rfloor_\alpha - \mathbf{s}_t(\epsilon)\rp \geq \alpha
   \end{equation}
   Note that by construction, $\abs{\lfloor f \rfloor_\alpha - f} \leq \frac \alpha 2$.  Thus for this $f \in \F$,
   \begin{align}
       \epsilon_t\lp  f(\mathbf{w}_t(\epsilon) - \mathbf{s}_t(\epsilon)\rp &\geq \epsilon_t\lp f(\mathbf{w}_t(\epsilon) - \lfloor f(\mathbf{w}_t(\epsilon)\rfloor_\alpha\rp + \epsilon_t\lp \lfloor f(\mathbf{w}_t(\epsilon)\rfloor_\alpha - \mathbf{s}_t(\epsilon)\rp \\
       &\geq \alpha - \frac \alpha 2 \geq \frac \alpha 2
   \end{align}
   Thus $\z$ also shatters $\F$ at scale $\alpha$.
\end{proof}
We are now ready to prove Theorem \ref{thm:fatshattering}:
\begin{proof}[Proof of Theorem \ref{thm:fatshattering}]
	We claim that for $\alpha = \frac 67 \delta$ and with the notation defined above, $N'(\F, \delta, \z) \leq N'\lp\mathcal{G}, \frac 23, \z\rp$.  To see this, note that if $\vbf$ such that $\norm{\vbf(\epsilon) - \lfloor f \rfloor_\alpha(\z(\epsilon))} \leq \frac 47 \delta n$, then we see that the triangle inequality implies
	\begin{equation}
	    \norm{\vbf(\epsilon) - f(\z(\epsilon)} \leq \norm{\vbf(\epsilon) - \lfloor f \rfloor_\alpha(\z(\epsilon))} + \norm{f(\z(\epsilon)) - \lfloor f \rfloor_\alpha(\z(\epsilon))} \leq \frac 47 \delta n + \frac 37 \delta n = \delta n
	\end{equation}
	Thus if $\mu$ is a $\frac 47 \delta$ fractional cover on $\lfloor\F\rfloor_\alpha$, then it is a $\delta$-cover on $\F$ for $\alpha = \frac 67 \delta$.  Scaling by $\alpha = \frac 67 \delta$ shows that if $N'\lp\lfloor\F\rfloor_\alpha, \frac 47 \delta, \z\rp = N'\lp \mathcal{G}, \frac 23, \z\rp$.  Now, by construction, $\mathcal{G}$ takes values in $\left\{0,1,\dots, \left\lfloor\frac 2\alpha\right\rfloor\right\}$.  Thus, setting $m = \frac 2\alpha = \frac 7{3\delta}$ and applying Proposition \ref{prop:discretefatshattering} gives
	\begin{equation}
	    N'(\F, \delta, \z) \leq N'\lp\mathcal{G}, \frac 23, \z\rp \leq C \lp\frac{36 e }{7} \lp \frac 7{3 \delta}\rp^3 \rp^{\fat_2(\mathcal{G}} \leq C \lp\frac{4e}{147 \delta^3} \rp^{\fat_{\frac 67 \delta}(\F)}
	\end{equation}
	as desired.
\end{proof}

\section{Proof of Proposition \ref{prop:fat-integral}}\label{app:fat-integral}

We start by proving Proposition~\ref{prop:fatshattering-gen}. We first consider the case $\mathcal{Y} = \{1,\dots,m\}$ and prove:
\begin{proposition}\label{prop2-gen}
	Suppose that $\mathcal{F}$ is a collection of functions from $\mathcal{Z}$ to $\mathcal{Y} =  \{1,\dots,m\}$ and $\delta \ge 2$. Then 
	\begin{equation}\label{eq:improved-bnd-cover}
	N'(\mathcal{F}, \delta) \leq \prod_{i=1}^{\log m} \left(\frac {C^i}{\delta}\right)^{\fat_{2^i}(\F)},
	\end{equation}
	where $C>0$ is a universal constant.
\end{proposition}
Below, we present the proof of Proposition \ref{prop2-gen} and then we prove Proposition \ref{prop:fatshattering-gen} and finally Proposition \ref{prop:fat-integral}.
First, we can assume that $m$ is an integer power of $2$.
Given a tree $\z$ and a concept class $\F$, we construct the fractional cover $\mu=\mu_{\z,\F}$ inductively. We define $\mu$ to be a sub-probability distribution, namely, each element has a non-negative measure and the sum of measures is at most $1$. By normalizing $\mu$, one can derive a proper probability measure.
	
Suppose we have a binary tree $\z$. If $\z$ is an empty tree then $\mu_{\z,\F}$ has a unit mass on the empty tree. Otherwise, define by $\z_0$, $\z^{(-1)}$ and $\z^{(1)}$ its root and left and right sub-trees, respectively (if the tree has depth zero then $\z^{(\pm 1)}$ are empty trees). For any concept-class $\mathcal{G}$, define by $\mu_{\z,\mathcal{G},i}$ the measure on $\mathcal{Y}$-valued trees $\vbf$ such that the root is labeled by $i$ almost surely and the left and right sub-trees are sampled independently from $\mu_{\z^{(-1)},\mathcal{G}}$ and $\mu_{\z^{(1)},\mathcal{G}}$.

Define $\mathcal{F}_j = \left\{f \in \mathcal{F}: f(\z_0) = j \right\}$ for $j \in [m]$.
For $j=1,\dots,m$ denote \[i(j) = \max\{i \colon 1 \le i \le \log_2 m, \fat_{2^i}(\F_j) < \fat_{2^i}(\F)\}
	\]
	with $i(j) = 0$ if $\fat_{2^i}(\F_j) = \fat_{2^i}(\F)$ for all $i$.  Let $j^*$ be an (arbitrarily chosen) minimizer of $i(j)$.

We define $\mu_{\z,\F}$. which is parametrized by some parameter $p$ to be set later, as follows: For any binary tree $\vbf$, define
\[
\mu_{\z,\F}(v) = (1-p)\mu_{\z,\F,j^*}(v) + \sum_{j=1}^{m} \lambda_j \mu_{\z,\F_j,j}(v), \quad
\text{where } \lambda_j = 4^{-i(j)-1}p.
\]
Observe that the first measure in the sum, $\mu_{\z,\F,j^*}$, corresponds to the class $\F$ while the remaining measures, $\mu_{\z,\F_j,j}$ for $j \in [m]$, correspond to the restricted concept classes $\F_j$. Second, note that the sum of mixture coefficients, $1-p+\sum_j \lambda_j$, does not necessarily equal $1$. Since, we define a sub-probability measure, it suffices to show that this sum is upper bounded by $1$.  To do this, we use Lemma \ref{lem:inductive-fat} from Appendix \ref{app:fatshattering}.

By Lemma \ref{lem:inductive-fat}, there are at most $2^{i+1}$ classes with $i(j) = i$, hence,
	\[
	1-p+ \sum_{j=1}^m \lambda_j
	\le 1-p + \sum_{i=0}^{\log_2 m} |\{ j \colon i(j) = i \}| \frac{p}{4^{i+1}}
    \le 1-p + \sum_{i=0}^{\log_2 m} \frac{2^{i+1}p}{4^{i+1}}
    = 1-p + \sum_{i=0}^{\log_2 m} \frac{p}{2^{i+1}}
    \le 1,
	\]
	which implies that we defined a sub-probability distribution as required.

	We claim that with an appropriate choice of $p$, the above construction yields a $\delta$-fractional cover of the correct size.
	In order to prove the claim, we introduce some notation. For a depth $n$ binary tree $\z$, a vector $\vec{k}=(k_1,\dots,k_{\log_2 m})$ a concept class $\F$, a function $f$ and a path $\epsilon$, denote the following set of $[m]$-labelled trees that contains all trees that are close to $f(z)$ on some path $\epsilon$, where the closeness is measured with respect to $\vec{k}$:
	\[
	E(\vec{k},f,\epsilon,z) = \left\{
	    \vbf \colon \forall i=1,\dots,\log_2 m, |\{t \colon 0 \le t \le n, |\vbf_t(\epsilon)-f(\z_t(\epsilon))|\ge 2^i\}|\leq \sum_{i'=i}^{\log_2 m} k_{i'}
	\right\}.
	\]
	As we shall show later, for all $\vbf \in E(\vec{k}, f,\epsilon,z)$ it holds that $\|\vbf(\epsilon)-f(\z(\epsilon))\|_2^2 \le C\sum_i 4^i k_i/n$.
	Next, we define:
	\begin{equation}
	\phi(\vec{k}, \mathcal{F}, \z) = \min_{\substack{f \in \mathcal{F} \\ \epsilon \in \{\pm1\}^n}} \mu\left(E(\vec{k},f,\epsilon,\z)\right)
	\end{equation}
    Where $\mu$ is as constructed above.  Notice that $\phi(\vec{k},\F,\z)$ can be used to bound the fractional covering numbers of $\F$, since $E(\vec{k},f,\epsilon,\z)$ contains trees that approximate $f$ on $\epsilon$. Lastly, for a vector $\vec{d}=(d_1,\dots,d_{\log_2 m})$, the following definition bounds the minimal value of $\varphi$ taken over concept-classes $\F$ whose fat covering numbers are bounded in terms of $\vec{d}$, as defined below:
	\begin{equation}
	    \varphi(\vec{k}, \vec{d}, n) = \min_{\substack{\F,\z \colon\\ \forall i=1,\dots,\log_2 m,\fat_{2^i}(\mathcal{F}) \leq d_i \\ \text{depth of } z \leq n}} \phi(\vec{k}, \mathcal{F}, \z).
	\end{equation}
	Here, $n\ge -1$ where $n=-1$ corresponds to the empty tree. 
	Define $e_i$ as the vector with $1$ in coordinate $i$ and zeros otherwise. One can prove the following inductive formula:
		\begin{lemma}
		For any $n\ge 0$,
		\[
		\varphi(\vec{k},\vec{d},n) \ge \min_{i=0,\dots,\log_2 m} \varphi_i(\vec{k},\vec{d},n).
		\]
		where $\varphi_0(\vec{k},\vec{d},n) = (1-p)\varphi(\vec{k},\vec{d},n-1)$ and for all $i=1,\dots,\log_2 m$
		\[
		\varphi_i(\vec{k},\vec{d},n) = \begin{cases}
		   4^{-i-1}p \varphi(\vec{k},\vec{d}-e_i,n) + (1-p) \varphi(\vec{k}-e_i, \vec{d}, n-1) & k_i,d_i>0 \\
		   4^{-i-1}p \varphi(\vec{k},\vec{d}-e_i,n) & k_i=0, d_i>0 \\
		   \infty & d_i=0
		\end{cases}.
		\]
		\end{lemma}
		\begin{proof}
		Fix $\z$ and $\F$ such that $\fat_{2^i}(\F) \le d_i$ for all $i$. Fix $f$ and $\epsilon$, denote $j = f(\z_0)$. We would like to argue that $\mu(E(\vec{k},f,\epsilon,\z)) \ge \varphi_{i(j)}(\vec{k},\vec{d},n)$. 
		For any tree $\vbf$, denote by $\vbf^{(\epsilon)}$ the sub-tree of $\vbf$, rooted by a child of the root of $\vbf$, that contains the sub-path $\epsilon_{1:n}$.
		To derive the proof, we divide into cases:
		\begin{enumerate}
		    \item If $i(j)=0$: 
		    This follows from the fact that $|j-j^*|< 2^{i(j)+1}=2$ and from the definition of $E(\vec{k},f,\epsilon,\z)$, one has:
		    \begin{equation} \label{eq:31}
		    E(\vec{k},f,\epsilon,\z)
    		\supseteq \{\vbf \colon \vbf^{(\epsilon)} \in E(\vec{k},f,\epsilon_{1:n}, \z^{(\epsilon)}), \mathrm{root}(\vbf) = j^* \}.
		    \end{equation}
		    As a consequence,
		    \begin{equation}\label{eq:41}
    	    \mu_{\z,\F,j^*}(E(\vec{k},f,\epsilon,\z)) \ge \mu_{\z^{(\epsilon)},\F}(E(\vec{k},f,\epsilon_{1:n},\z^{(\epsilon)})) \ge \phi(\vec{k},\F,\z^{(\epsilon)}).
    	    \end{equation}
		    Therefore, 
		    \begin{align}
		        \mu(E(\vec{k},f,\epsilon,\z))
		        &\ge (1-p) \mu_{\z,\F,j^*}(E(\vec{k},f,\epsilon,\z)) \\
		        &\ge (1-p)\phi(\vec{k},\F,\z^{(\epsilon)}) \\
		        &\ge (1-p)\varphi(\vec{k},\vec{d},n-1) \\
		        &= \varphi_0(\vec{k},\vec{d},n).
		    \end{align}
		    
		    \item If $i(j)>0$ and $k_{i(j)},d_{i(j)}>0$: 
		    First, notice that 
		    \begin{equation} \label{eq:30}
		    E(\vec{k},f,\epsilon,\z)
    		\supseteq \{\vbf \colon \vbf^{(\epsilon)} \in E(\vec{k},f,\epsilon_{1:n},\z^{(\epsilon)}), \mathrm{root}(\vbf) = j \}.
		    \end{equation}
		    This implies that
		    \begin{equation}\label{eq:40}
        	\mu_{\z,\F_j,j}(E(\vec{k},f,\epsilon,\z)) \ge \mu_{z^{(\epsilon)},\F_j}(E(\vec{k},f,\epsilon_{1:n},\z^{(\epsilon)})) \ge \phi(\vec{k},\F_j,\z^{(\epsilon)}).
        	\end{equation}
        	Further, using the fact that $i(j)>0$, one has
        	\begin{equation}\label{eq:32}
    		E(\vec{k},f,\epsilon,\z)
    		\supseteq \{\vbf \colon \vbf^{(\epsilon)} \in E(\vec{k}-e_{i(j)},f,\epsilon_{1:n},\z^{(\epsilon)}), \mathrm{root}(\vbf) = j^* \},
    	    \end{equation}
    	    which implies that
    	    \begin{equation}\label{eq:42}
    	    \mu_{\z,\F,j^*}(E(\vec{k},f,\epsilon,\z)) \ge \mu_{\z^{(\epsilon)},\F}(E(\vec{k}-e_{i(j)},f,\epsilon_{1:n},\z^{(\epsilon)})) \ge \phi(\vec{k}-e_{i(j)},\F,\z^{(\epsilon)}).
    	\end{equation}
		    From \eqref{eq:40} and \eqref{eq:42},
		    \begin{align}
		    \mu(E(\vec{k},f,\epsilon,\z))
		    &\ge (1-p) \mu_{\z,\F,j^*}(E(\vec{k}-e_{i(j)},f,\epsilon,\z)) + 4^{-i(j)-1}p \mu_{\z,\F_j,j}(E(\vec{k},f,\epsilon,\z))\notag\\
		    &\ge (1-p)\varphi(\vec{k}-e_{i(j)},\F, \z^{(\epsilon)}) + 4^{-i(j)-1}p\phi(\vec{k},\F_j,\z^{(\epsilon)}).\label{eq:734}
		    \end{align}
		    From the definition of $i(j)$, one has that $\fat_{2^{i(j)}}(\F_j) < \fat_{2^{i(j)}}(\F)$. Hence, from \eqref{eq:734},
		    \begin{align}
		    \mu(E(\vec{k},f,\epsilon,\z))
		    &\ge (1-p)\varphi(\vec{k}-e_{i(j)},\vec{d},n-1) + 4^{-i(j)-1}p\varphi(\vec{k},\vec{d}-e_{i(j)}, n-1) \\
		    &= \varphi_{i(j)}(\vec{k},\vec{d},n).
		    \end{align}
		    \item If $i(j)>0$, $k_{i(j)}=0$ and $d_{i(j)}>0$: Using a similar argument as in the previous case,
		    \begin{align*}
		    \mu(E(\vec{k},f,\epsilon,\z))
		    &\ge 4^{-i(j)-1}p \mu_{\z,\F_j,j}(E(\vec{k},f,\epsilon,\z))
		    \ge 4^{-i(j)-1}p\phi(\vec{k},\F_j,\z^{(\epsilon)})\\
		    &\ge 4^{-i(j)-1}p\varphi(\vec{k},\vec{d}-e_{i(j)}, n-1) = \varphi_{i(j)}(\vec{k},\vec{d},n).
		    \end{align*}
		    \item If $i(j)>0$ and $d_{i(j)} = 0$: this cannot hold, since $0 \le \fat_{2^{i(j)}}(\F_j) < \fat_{2^{i(j)}}(\F) \le d_{i(j)}$.
		\end{enumerate}
		\end{proof}
	
We proceed with bounding $\varphi(\vec{k},\vec{d},\vec{n})$:
\begin{lemma}\label{lem2}
	Let $\phi(\vec{k},\vec{d},n)$ be as defined in the above proof.  Then for all $0 < p < 1$, $n\ge -1$ and $k_1,\dots,k_{\log_2 m},d_1,\dots,d_{\log_2 m} \ge 0$,
	\begin{align}\label{eq:phi-bnd-multiscale}
	\phi(\vec k,\vec d,n) 
	&\geq (1 - p)^{n + 1 + \sum_{i=1}^{\log_2 m} k_i} \prod_{i=1}^{\log_2 m} \left(\frac{p}{4^{i+1}}\right)^{d_i} \binom{k_i+d_i}{d_i}.
	\end{align}
\end{lemma}
\begin{proof}
		We prove the result by induction on $n$. For the base of induction, we assume that $n=-1$. Then the left hand side of \eqref{eq:phi-bnd-multiscale} equals $1$, while the right hand side is at most $1$, as follows from the following argument: For all $i$,
		\begin{align}
		(1-p)^{k_i} \left(\frac{p}{4^{i+1}} \right)^{d_i} \binom{k_i+d_i}{d_i}
		&\le (1-p)^{k_i} p^{d_i} \binom{k_i+d_i}{d_i} \\
		&\le \sum_{\ell=0}^{k_i+d_i} (1-p)^{k_i+d_i-\ell} p^{\ell} \binom{k_i+d_i}{\ell}\\
		&= ((1-p)+p)^{k_i + d_i}
		= 1.
		\end{align}
	
	For $n \ge 0$, we would like to lower bound $\varphi_i(\vec{k},\vec{d},n)$ by the right hand side of \eqref{eq:phi-bnd-multiscale}, for all $i =0,\dots,\log n$. First, for $i=0$, this follows directly from the induction hypothesis. For $i\ge 1$, divide into cases. First, assume that $k_i,d_i > 0$. Then, by induction hypothesis
	\begin{align*}
	\varphi_i(\vec{k},\vec{d},n)
	&= \frac{p}{4^{i+1}} \varphi(\vec{k},\vec{d}-e_i,n) + (1-p) \varphi(\vec{k}-e_i, \vec{d}, n-1) \\
	&\geq (1 - p)^{n + \sum_{i'\ne i} k_{i'}} \prod_{i' \ne i} \left(\frac{p}{4^{i'+1}}\right)^{d_{i'}} \binom{k_{i'}+d_{i'}}{d_{i'}} \cdot\\ &\left(\frac{p}{4^{i+1}}\cdot (1-p)^{k_i}\left(\frac{p}{4^{i+1}}\right)^{d_{i}-1} \binom{k_{i}+d_{i}-1}{d_{i}-1}
	+ (1-p) \cdot (1-p)^{k_i-1}\left(\frac{p}{4^{i+1}}\right)^{d_{i}} \binom{k_{i}-1+d_{i}}{d_{i}}
	\right)\\
	&= (1 - p)^{n + 1 + \sum_{i=1}^{\log_2 m} k_i} \prod_{i=1}^{\log_2 m} \left(\frac{p}{4^{i+1}}\right)^{d_i} \binom{k_i+d_i}{d_i},
	\end{align*}
	using the equality
	\[
	\binom{n}{k} = \binom{n-1}{k}+\binom{n-1}{k-1}.
	\]
	Next, assume that $k_i=0$. Then,
	\begin{align*}
	\varphi_i(\vec{k},\vec{d},n)
	&= \frac{p}{4^{i+1}} \varphi(\vec{k},\vec{d}-e_i,n) \\
	&\geq \frac{p}{4^{i+1}} \cdot (1 - p)^{n + \sum_{i'\ne i} k_{i'}} \prod_{i' \ne i} \left(\frac{p}{4^{i'+1}}\right)^{d_{i'}} \binom{k_{i'}+d_{i'}}{d_{i'}} \cdot \left(\frac{p}{4^{i+1}}\right)^{d_{i}-1} \binom{d_{i}-1}{d_{i}-1} \\
	&= \cdot (1 - p)^{n + \sum_{i'\ne i} k_{i'}} \prod_{i' \ne i} \left(\frac{p}{4^{i'+1}}\right)^{d_{i'}} \binom{k_{i'}+d_{i'}}{d_{i'}} \cdot \left(\frac{p}{4^{i+1}}\right)^{d_{i}} \binom{d_{i}}{d_{i}} \\
	&= (1 - p)^{n + 1 + \sum_{i=1}^{\log_2 m} k_i} \prod_{i=1}^{\log_2 m} \left(\frac{p}{4^{i+1}}\right)^{d_i} \binom{k_i+d_i}{d_i}.
	\end{align*}
	Lastly, assume that $d_i=0$. Then, $\varphi_i(\vec{k},\vec{d},n) = \infty$, and the statement holds as well.
	\end{proof}
\begin{proof}[Proof of Proposition~\ref{prop2-gen}]
	Define $k_i = 8^{-i}\delta n$ for $i=1,\dots,\log_2 n$ and notice that it suffices to show that 
	\[
	\varphi(\vec{k}, \vec{d},n)
	\ge \prod_{i=1}^{\log_2 m} \left( c^i\delta\right)^{\fat_{2^i}(\F)}.
	\]
    Indeed, for any $f,\epsilon$, and any $\vbf\in E(\vec{k},f,\epsilon,\z)$, it holds that 
	\[
	\|\vbf(\epsilon)-f(\z(\epsilon))\|_2 = \sqrt{\frac{1}{n+1}\sum_{t=0}^n (\vbf_t(\epsilon)-f(\z_t(\epsilon)))^2}
	\le C \delta,
	\]
	for some universal constant $C>0$. This implies that the lower bound on $\varphi(\vec{k},\vec{d},n)$ can be directly translated to an upper bound on the fractional $\ell_2$ covering numbers.

	Applying Lemma \ref{lem2} with the parameter $p=d_1/n$, using $d_1 = \fat_2(\F) \ge \fat_{2^i}(\F)=d_i$ for all $i\ge 1$ and using the inequality $\binom{n}{k}\ge (n/k)^k$, we derive that
	\[
	\varphi(\vec{k},\vec{d},n)
	\ge (1-\frac{d_1}{n})^{2n} \prod_{i=1}^{\log n}
	\left(\frac{d_1}{4^{i+1}n}\right)^{d_i} \left(\frac{k_i}{d_i}\right)^{d_i}
	\ge c e^{-2d_1} \prod_{i=1}^{\log n} \left( \frac{\delta}{4\cdot 32^i} \right)^{d_i},
	\]
	as required.
\end{proof}

Using Proposition \ref{prop2-gen}, we can bound the fractional covering number of a real valued function class, with respect to its fat-shattering dimension.

\begin{proof}[Proof of Proposition \ref{prop:fatshattering-gen}]
    We apply Lemma \ref{lem:rakhlindiscretization} and replicate the proof of Theorem \ref{thm:fatshattering} found in Appendix \ref{app:fatshattering}.  With the notation from that section, we note that for any $\z$ by the results of that proof,
    \begin{align}
        N'(\F, \delta, \z) &\leq N'\lp \frac 1\alpha \lfloor \F\rfloor_\alpha, \frac{\delta}{\alpha} - \frac 12, \z\rp \\
         \fat_2\lp \frac 1\alpha \lfloor \F \rfloor_\alpha\rp&= \fat_{2 \alpha} \lp \lfloor \F \rfloor_\alpha \rp \leq \fat_\alpha(\F)
    \end{align}
    Letting $\alpha = \frac 23 \delta$ and plugging into the result of Proposition \ref{prop2-gen} gives us
    \begin{align}
        N'(\F, \delta, \z) &\leq N'\lp \frac 1\alpha \lfloor \F \rfloor_\alpha, 1, \z\rp \leq \prod_{i = 1}^{\log \frac c\delta} \lp \frac{C}{\frac{\delta}{\alpha} - \frac 12}\rp^{\fat_{2^i}\lp \frac 1\alpha \lfloor \F \rfloor_\alpha\rp} \\
        &\leq \prod_{i = 1}^{\log \frac c\delta} C^{\fat_{2^i c \delta}(\F)}
    \end{align}
    as desired.
\end{proof}

\begin{proof}[Proof of Theorem~\ref{prop:fat-integral}]
    We start by bounded in $\sqrt{\log N'(\F,\delta)}$ for a single value of $\delta$.
	Taking a logarithm in \eqref{eq:multiscale-bnd}, one obtains
	\begin{equation}
		\log N'(\mathcal{F}, \delta) \leq 
		C\sum_{i=1}^{\log(1/\delta)+C'} i\cdot\fat_{2^ic\delta}(\F).
	\end{equation}
	Using the subadditivity of the square root, one obtains:
	\begin{equation}
	    \sqrt{\log N'(\mathcal{F}, \delta)} \leq 
		\sqrt{C}\sum_{i=1}^{\log(1/\delta)+C'} \sqrt{i\cdot\fat_{2^ic\delta}(\F)}.
	\end{equation}
	Each term corresponding to $i$ can be bounded by
	\[
	\frac{1}{2^{i-1}c\delta}\int_{x=2^{i-1} c\delta}^{2^i c\delta} \sqrt{i \fat_{x}(\F)}
	\le 2 \int_{x=2^{i-1} c\delta}^{2^i c\delta} \frac{\sqrt{i\fat_{x}(\F)}}{x}
	\le C \int_{x=2^{i-1} c\delta}^{2^i c\delta} \frac{\sqrt{\log(Cx/\delta)\fat_{x}(\F)}}{x}.
	\]
	Hence, we derive that
	\begin{equation}
		\sqrt{\log N'(\mathcal{F}, \delta)} \leq C \int_{x=c\delta}^1 \frac{\sqrt{\log(Cx/\delta)\fat_{x}(\F)}}{x}
	\end{equation}
	
	Next, we combine the above bound by taking an integral, and switching the order of integration:
    \begin{align*}
    \int_{b}^1 \sqrt{\log N'(\F,\delta)}d\delta
    &\le C\int_{\delta=b}^1 \int_{x=c\delta}^1 \frac{\sqrt{\log(x/\delta)\fat_{x}(\F)}}{x}dx d\delta
    \\
    &= C\int_{x=cb}^1 \sqrt{\fat_x(\F)} \int_{\delta=b}^{x/c}  \frac{\sqrt{\log(Cx/\delta)}}{x} d\delta dx\\
    &\le C'\int_{x=cb}^1 \sqrt{\fat_x(\F)} dx.
    \end{align*}
\end{proof}

\section{Miscellaneous Proofs}\label{app:miscellany}
\begin{proof}[Proof of Proposition \ref{prop:lipschitzmajorizingmeasures}]
   By scaling, it suffices to consider $L = 1$.  Let $\widetilde{\mu} = \ell_\# \mu$ be the pushforward of $\mu$ by $\ell$.  Because $\ell$ is Lipschitz, we note that $\ell \circ B_\delta(f(\z), \epsilon) \subseteq B_\delta(\ell \circ f(\z), \epsilon)$.  By monotonicity of measures, we then have
   \begin{equation}
       \sqrt{\log \frac 1{\widetilde{\mu}(B_\delta(\ell \circ f(\z), \epsilon))}} \leq \sqrt{\log \frac 1{\widetilde{\mu}(\ell (B_\delta( f(\z), \epsilon))}} =\sqrt{\log \frac 1{\mu(B_\delta(f(\z), \epsilon))}}
   \end{equation}
   where the equality follows from the definition of the push-forward.  The result follows by taking an infimum over measures $\mu$.
\end{proof}

\begin{proof}[Proof of Lemma \ref{lem:fractionalvsclassicalcover}]
  Let $\vbf_1, \dots, \vbf_N$ be a $\delta$-cover and let $\mu$ be a a measure that takes $\vbf_j$ with probability $\frac 1N$.  Then by definition of the covering numbers, $\mu$ is a $\delta$-fractional cover of size $N$.  The result follows.
\end{proof}
\begin{proof}[Proof of Lemma \ref{lem:classicalfractionalcover}]
    The upper bound follows from Lemma \ref{lem:fractionalvsclassicalcover}.  Fix a $\z$.  As we are in the offline world, we have packing-covering duality (see, for example, \cite[Lemma 5.12]{van2014probability}).  Thus, it suffices to show that $D\lp\F, 2 \delta\rp \leq N'(\F, \delta)$.  To see this, consider a set of points $x_1, \dots, x_N$ that are $2\delta$-packed and consider the balls $B_\delta(x_i)$.  By definition of a fractional cover, if $\mu$ is such, then $\mu(B_\delta(x_i)) \geq \frac 1\gamma$ for all $i$.  by the fact that the $x_i$ are packed, the sets $B_\delta(x_i)$ are pairwise disjoint and so by additivity of the measure and the fact that it is a probability measure,
    \begin{equation}
        1 \geq \mu\lp \bigcup_{1 \leq i \leq N} B_{\delta}(x_i)\rp \geq N \mu\lp B_\delta(x_1)\rp \geq \frac N\gamma
    \end{equation}
    the result follows.
\end{proof}
\begin{proof}[Proof of Proposition \ref{prop:majorizingmeasure}]
    A majorizing measure is multi-scale, while a fractional cover is single-scale; to turn the latter into the former, we consider the following mixture distribution.  If the integral of the fractional cover is infinite then there is nothing to prove; thus, assume that the this quantity is finite.  Suppose $\alpha = 0$.  Let $\delta_j$ be the smallest $\delta$ such that $N'(\mathcal{V}, \delta) \leq 2^{2^j}$ and let $\mu_j$ be an optimal $\delta_j$-fractional cover.  Let
	\begin{equation}
		\mu = c \sum_j j^{-2} \mu_j
	\end{equation}
	Note that the above sum is over all $j$ sufficiently large by the assumption that the integral is finite and we let $c$ be a normalizing constant.  Now, we have for any $\vbf, \epsilon$,
	\begin{align}
		I_\mu(\vbf, \epsilon) = \int_0^1 \sqrt{\log \frac 1{\mu(B_\delta(\vbf,\epsilon))}} d \delta \leq \sum_j \int_{\delta_j}^{\delta_{j-1}} \sqrt{\log \frac 1{\mu(B_\delta(\vbf,\epsilon))}} d \delta 
	\end{align}
	By construction, for any $\delta \geq \delta_j$,
	\begin{align}
		\mu(B_\delta(\vbf, \epsilon)) \geq \mu(B_{\delta_j}(\vbf,\epsilon)) \geq \frac 1{j^2}\mu_j(B_{\delta_j}(\vbf, \epsilon)) \geq \frac 1{j^2 2^{2^j}} \geq \frac 1{2^{2\times2^j}}
	\end{align}
	Thus we have
	\begin{align}
		\int_{\delta_j}^{\delta_{j-1}} \sqrt{\log \frac 1{\mu(B_\delta(\vbf,\epsilon))}} d \delta \leq c \delta_{j-1} 2^{\frac j2} \leq c' \delta_{j-1} 2^{ \frac{j-1}{2}}
	\end{align}
	Thus,
	\begin{align}
		I_\mu(\vbf, \epsilon) \leq \sum_j c' \delta_{j-1} 2^{ \frac{j-1}{2}} \leq C \int_0^1 \sqrt{\log N'(\mathcal{V}, \delta)} d \delta
	\end{align}
	If $\alpha > 0$, simply cut off $j$ such that $\delta_j \geq \alpha$.  The same technique applies, concluding the proof.
\end{proof}

\begin{proof}[Proof of Corollary \ref{cor:uniformconcentration}]
    By \cite[Lemma 4]{rakhlin2015sequential}, we have for any $x > 0$
    \begin{align}\label{eq:stochasticdomination}
        \pp_\epsilon\lp \sup_{f \in \F} \abs{\frac{1}{n} \sum_{t = 1}^n f(Z_t) - \ee[f(Z_t)| \mathcal{A}_t]} \geq x  \rp \leq 4 \pp_\epsilon\lp \abs{\frac{1}{n} \sum_{t = 1}^n \epsilon_t f(\z_t(\epsilon))}  \rp
    \end{align}
    By Theorem \ref{thm:chaining}, we can control the right hand side of \eqref{eq:stochasticdomination} and conclude the proof.
\end{proof}

\begin{proof}[Proof of Corollary \ref{cor:rademacher}]
   The upper bound follows immediately from setting $\alpha = 0$ in the second statement of Theorem \ref{thm:chaining} and then bounding the resulting integral by Proposition \ref{prop:fat-integral}.  
   
   For the lower bound, we first need to show that if $\F$ is $(c,p)$-bounded, then
   \begin{equation}
       \int_0^1 \sqrt{\fat_\delta(\F)} d \delta \leq \frac{2 \sqrt c}{1 - 2^{\frac p2 - 1}}
   \end{equation}
   To do this, we note that the function $\delta \mapsto \fat_\delta(\F)$ is monotone non-increasing.  Thus we have
   \begin{align}
       \int_0^1 \sqrt{\fat_\delta(\F)} d \delta &= \sum_{ j = 1}^\infty \int_{2^{-j}}^{2^{-j + 1}} \sqrt{\fat_\delta(\F)} d \delta \leq \sum_{j  =1}^\infty 2^{-j} \sqrt{\fat_{2^{-j}}(\F)} \\
       &\leq \sum_{j = 1}^\infty 2^{-j} \sqrt{c 2^{jp}} \leq \frac{2 \sqrt c}{1 - 2^{\frac p2 - 1}}
   \end{align}
   Now, noting that $\fat_1(\F) \geq 1$, we see that the above implies that for any $ \alpha < 1$, 
   \begin{equation}
       \int_0^1 \sqrt{\fat_\delta(\F)} d \delta \leq \frac{2 \sqrt c}{1 - 2^{\frac p2 - 1}} \leq \frac{2 \sqrt c}{1 - 2^{\frac p2 - 1}} \sup_{\delta > \alpha} \delta \sqrt{\fat_\delta(\F)}
   \end{equation}
   Note that the definition of a $(c,p)$-bounded class ensures that the right hand side of Proposition \ref{prop:fat-integral} tends to a finite limit as $b \downarrow 0$; call this finite limit $A$.  To get the lower bound, we apply \cite[Lemma 2]{rakhlin2015sequential}, which says that
   \begin{equation}
       \sup_{\delta > 0} \delta \sqrt{\frac{\min(\fat_\delta(\F), n)}{32}} \leq \sqrt n \R_n(\F)
   \end{equation}
   and, moreover, that if $\delta > 2 \R_n(\F)$, then $\fat_\delta(\F) < n$.  
   Thus, using the upper bound on $\R_n(\F)$ just proven, we see that
   \begin{align}
       \sqrt n \R_n(\F) &\geq C' \sup_{\delta > 2 \R_n(\F)} \delta \sqrt{\fat_\delta(\F)} \geq C' \sup_{\delta > \frac {c'}{\sqrt n}\int_0^1 \sqrt{\fat_\delta(\F)} d \delta } \delta \sqrt{\fat_\delta(\F)} \\
       &\geq C' \sup_{\delta > \frac{c' A}{\sqrt n}}  \delta \sqrt{\fat_\delta(\F)} \geq C' \frac{1 - 2^{\frac p2 - 1}}{2 \sqrt c} \int_0^1 \sqrt{\fat_\delta(\F)} d \delta
   \end{align}
   The result follows.
\end{proof}

\begin{proof}[Proof of Corollary \ref{cor:contraction}]
   By scaling it suffices to consider the case $L = 1$.  By Theorem \ref{thm:chaining} we have
   \begin{equation}
       \R_n(\ell \circ \F) \leq \frac C{\sqrt n} \int_0^1 \sqrt{\log N'(\ell \circ \F, \delta, \z)} d \delta
   \end{equation}
   It is immediate from the definition of the fractional covering number that $N'(\ell \circ \F, \delta, \z) \leq N'(\F, \delta, \z)$.  Applying Proposition \ref{prop:fat-integral} and Corollary \ref{cor:rademacher} concludes the proof.
\end{proof}

\end{document}